\RequirePackage{fix-cm}
\documentclass[twocolumn,natbib]{svjour3}          
\smartqed  
\usepackage{graphicx}
\usepackage{graphics} 
\usepackage{epsfig} 
\usepackage{epstopdf}
\usepackage{amsmath} 
\usepackage{amssymb}  
\usepackage[scriptsize]{subfigure}
\usepackage{enumerate}
\usepackage{color}
\usepackage[vlined,linesnumbered,ruled]{algorithm2e}
\usepackage{epstopdf}
\graphicspath{{figures/}}

\usepackage{booktabs}

\DeclareMathOperator*{\argmin}{argmin}

\usepackage{amsthm}
\theoremstyle{plain}
\newtheorem{assumption}{Assumption}
\theoremstyle{remark}

\newcommand{\qvect}{\mathbf{q}}
\newcommand{\uvect}{\mathbf{u}}
\newcommand{\RRTMP}{${\tt {RRT^\star \text{\underline{\hspace{0.2cm}}} MotionPrimitives}}$}

\newcommand{\setLUT}{\mathcal{Z}}
\newcommand{\costLUT}{\mathcal{C}}

\newcommand{\BSnote}[1]{{\textcolor{black}{#1}}}

\begin{document}

\title{Sampling-based optimal kinodynamic planning with motion primitives
}


\author{Basak Sakcak$^{1}$ \and Luca Bascetta$^{1}$
	\and Gianni Ferretti$^{1}$ \and Maria Prandini$^{1}$}

\authorrunning{Basak Sakcak et al.} 

\institute{$^{1}$The authors are with Politecnico di Milano, Dipartimento di Elettronica, Informazione e Bioingegneria, Piazza L. Da Vinci 32, 20133, Milano, Italy
	{\tt name.surname@polimi.it}}

\date{Received: date / Accepted: date}

\maketitle

\begin{abstract}
This paper proposes a novel sampling-based motion planner, which integrates in  RRT$^\star$ (Rapidly exploring Random Tree star) a database of pre-computed motion primitives to alleviate its computational load and allow for motion planning in a dynamic or partially known environment.
The database is built by considering a set of initial and final state pairs in some grid space, and determining for each pair an optimal trajectory that is compatible with the system dynamics and constraints, while minimizing a cost.
Nodes are progressively added to the tree {of feasible trajectories in the RRT$^\star$ algorithm} by extracting at random a sample in the gridded state space and selecting the best obstacle-free motion primitive in the database that joins it to an existing node.
The tree is rewired if some nodes can be reached from the new sampled state through an obstacle-free motion primitive with lower cost.
The computationally more intensive part of motion planning is thus moved to the preliminary offline phase of the database construction {at the price of some performance degradation due to gridding. Grid resolution can be tuned so as to compromise between (sub)optimality and size of the database. The planner is shown to be }asymptotically optimal as the grid resolution goes to zero and the number of sampled states grows to infinity.
\keywords{Optimal sampling-based planning \and Kinodynamic planning \and Motion Primitives}
\end{abstract}

\section{Introduction}
Motion planning is one of the fundamental problems in robotics, and consists of guiding the robot from an initial state to a final one along a collision-free path.

For many robots, focusing only on kinematics could result in collision-free paths that are impossible to be executed by the actual system. In particular, for systems that are differentially constrained, such a decoupled approach makes it difficult to turn a collision-free path into a feasible trajectory. In order to tackle this problem, \cite{donald1993kinodynamic} proposed the idea of \emph{kinodynamic planning}, which combines the search for a collision-free path with the underlying dynamics of the system, so that the resulting trajectory would be feasible.

For most robotic applications, the solution to the planning problem should not only be feasible and collision-free, but also satisfy some properties such as, e.g., reaching the goal in minimum time, minimizing the energy consumption, and maximizing safety. These additional requirements have shifted the focus from simply designing collision-free and feasible trajectories to finding optimal ones.

This paper deals with \emph{optimal kinodynamic motion planning} for systems with complex dynamics and subject to constraints.

\subsection{Literature review}

A significant amount of work in the robotics community has then been dedicated to the problem of kinodynamic planning so as to determine a trajectory that fulfills the differential constraints arising from the dynamics of the robot. Solving this problem is complex, in general, since it requires a search in the state space of the robot, which often implies a higher-dimensional search space compared to a pure kinematic planning.

Most of the motion planners can be classified under the two categories of \emph{exact} and \emph{sampling-based} methods, \cite{lavalle2006planning}.
The former looks for a solution in the continuous state space, while the latter samples this space redefining it as a graph where nodes are connected via edges representing local trajectories between sampled states.

Exact methods are said to be \emph{complete}, since they terminate in finite time with a solution, if one exists, and return the nonexistence otherwise. However, with the exact approaches even the simplest problem is PSPACE hard (\citealp*{reif1979complexity}).
Exact methods require that the obstacles are represented explicitly in the state space, dramatically increasing the problem complexity. However, they can provide practical solutions for problems that are characterized by a low dimensional state space, or for which a low dimensional obstacle representation can be adopted. Obstacles introduce non convex constraints in the admissible state space, and make the problem of computing an optimal trajectory hard. In particular, most of the algorithms, that are gradient based, can only find a solution in the same homotopy class of the initial guess (\citealp*{lavalle2006planning}). There has been a significant progress to address this issue and, in particular, the ideas of dividing the global optimal trajectory planning problem into simpler subproblems and of using numerical optimization to compute locally optimal trajectories have been explored in, e.g., \citep{park2015homotopy,kuderer2014online}.

Sampling-based approaches emerged to handle systems with high dimensional state spaces, and they became  the most popular approaches in the planning literature techniques, representing the practical way to tackle the problem.
The basic idea is to sample states (nodes) in the continuous state space and connect these nodes with trajectories in the collision-free space, building a roadmap in the form of a graph or a tree of feasible trajectories. These algorithms avoid an explicit representation of obstacles by using a collision check module that allows to determine the feasibility of a tentative trajectory. They are not complete, but they satisfy the \emph{probabilistic completeness} property, i.e., they return a solution with a probability converging to one as the number of samples grows to infinity, if such a solution exists. 
\\
Probabilistic Road Map (PRM), introduced by \citet{kavraki1996probabilistic}, and Rapidly exploring Random Trees (RRT), introduced by \citet{lavalle2001randomized}, were the first popular sampling-based planners. PRM first creates a graph in the free configuration space by randomly sampling nodes and connecting them to the already existing ones in the graph using a local planner. The graph can then be used to answer multiple queries, where in each query a start node and a goal node are added to the graph and a path connecting the two nodes is looked for. RRT, on the other hand, incrementally builds a tree starting from a given node, returning a solution as soon as the tree reaches the goal region, hence providing a fast on-line implementation.
In all the different formulations of sampling-based planners, a steering function is required to design a trajectory (edge in the tree terminology) connecting two nodes of the tree.

Considering the quality of the solution, an important progress has been made with the introduction of RRT$^\star$ (Rapidly exploring Random Tree star) and PRM$^\star$ (Probabilistic Road Map star), which have been proven to be \emph{asymptotically optimal}, i.e., the probability of finding an optimal solution, if there exists one, converges to 1 as the tree cardinality grows to infinity, \citep{karaman2011sampling}.  The main idea of these algorithms is to ensure that each node is connected to the graph optimally, possibly rewiring the graph by testing connections with pre-existing nodes that are in a suitably defined  neighborhood.
The same strategy applies to kinodynamic planning \citep{karaman2010optimal} as well, with the additional difficulty that  when optimality is required,  implementing the steering function involves solving a \emph{two point boundary value problem} (TPBVP), which is computationally challenging especially when dealing with complex dynamics, such as for non-holonomic robots, in presence of actuation constraints.
In the context of kinodynamic planning RRT$^\star$ and PRM$^\star$ cannot be considered in the same way. In fact, PRM$^\star$ is limited to symmetric costs and to those systems for which the cost associated to a TPBVP is conserved when the boundary pairs are swapped. Note that, nonholonomic systems do not belong to this class.\\
Considering instead RRT$^\star$, it must be noticed that for various dynamical systems, such as non-holonomic vehicles, the presence of kinodynamic constraints makes the constrained optimization problem that the steering function has to solve extremely complex.\\
To deal with this computational complexity, some effort has been made towards developing effective steering functions for different types of dynamical systems. \citet{webb2013kinodynamic} have obtained the closed-form analytical solution for a minimum time minimum energy optimization problem for systems with linear dynamics, and extended it to non-linear dynamics using first-order Taylor approximation. Other works \citep{perez2012lqr,goretkin2013optimal} have focused on approximating the solution for systems with linearizable dynamics, by locally linearizing the system and applying linear quadratic regulation (LQR).\\
Some recent attempts have been made towards optimality without formulating and solving a  TPBVP, as well. For example, algorithms like Stable Sparse RRT$^\star$ (SST$^\star$) \citep{li2016asymptotically} have proved asymptotic optimality given access only to a forward propagation model. The idea is to iteratively sample a set of controls and final times instead of explicitly solving the {BVP}. Similarly, a variant of RRT$^\star$ \citep{hwan2011anytime} uses a shooting method without a steering function to improve the solution by pruning branches from the tree. If a sampled node has a lower cost compared to another one that is close by and that shares the same parent, the pre-existing node is pruned from the tree and its branches are connected to the newly sampled node, or they are pruned completely if they are not collision-free. This approach generates feasible but inherently suboptimal solutions. Other works on extending RRT$^\star$ to handle kinodynamic constraints include limiting the volume in the state space from which nodes are selected by tailoring it to the considered dynamical system in order to improve computational effectiveness \citep{karaman2013sampling}.
\\
Nevertheless, solving the TPBVP for an arbitrary nonlinear system remains challenging and typically calls for numerical solvers. The algorithms that account for a nonlinear optimization tool, like for example {ACADO} toolkit \citep{Houska2011a}, {GPOPS-II} \citep{patterson2014gpops}, etc., commonly use \emph{Sequential Quadratic Programming} (SQP) \citep{boggs1995sequential} for solving the TPBVPs numerically, and embed it as a subroutine in the sampling based planning framework such as in RRT* \citep{stoneman2014embedding} or in Batch-Informed-Trees star (BIT$^\star$) \citep{gammell2014bit}.

A different class of algorithms, aiming at optimality, is based on graph search and adopt a gridding approach. The main idea is to discretize the state space, building a grid, and compute a graph. The motion planning problem is then recast into finding the best sequence of motions by traversing this graph with an optimal search algorithm like A$^\star$ \citep{pearl1984heuristics}. This graph is often represented by a state lattice, a set of states distributed in a regular pattern, where the connections between states are provided by feasible/optimal trajectories \citep{pivtoraiko2009differentially}. \citet{likhachev2009planning} improved the idea of state lattice by using a multi-resolution lattice such that the portion of the graph that is close to either the initial or the goal state has a higher resolution than the other parts. These approaches have been successfully applied to several robotic systems and found to be effective for dynamic environments \citep{likhachev2009planning,dolgov2010path}. However, these algorithms are \emph{resolution optimal}, such that the optimality is guaranteed up to the grid resolution. Furthermore, their computational effectiveness is highly related to the resolution: the finer is the grid, the higher the branching factor, and thus the computational time and the required memory to execute a graph traversal algorithm.

\subsection{Contribution of the paper}

The main contribution of this paper is proposing an algorithm, called \RRTMP, which extends  RRT$^\star$ by introducing a database of pre-computed motion primitives in order to avoid the online solution of a constrained TPBVP for the edge computation. The database is composed of a set of trajectories, each one connecting an initial state to a final one in a suitably defined grid. By sampling in the gridded state space, the implementation of the steering function adopted for growing and rewiring the RRT$^\star$ tree reduces to the search of a motion primitive in a pre-computed Look Up Table (LUT).\\
The proposed approach is applicable to any dynamical system described by differential equations and subject to analytical constraints, for which edge design can be formulated as a TPBVP. Notably, when a model of the robot is not available, the database can be derived directly from experimental trajectories.

The main difference of the proposed approach, with respect to existing algorithms that use a database of pre-computed trajectories, e.g., search based approaches as in \citep{pivtoraiko2009differentially}, is that it leverages on a dynamic tree whose size depends only on the number of samples, but not on the number of motion primitives that affects the accuracy in the approximation of the robot kinematic and dynamic characteristics. As a consequence, memory consumption to store the tree and computation time to determine a solution on a given tree can be bounded selecting an appropriate maximum number of samples, without introducing undesired constraints in the robot action space.\\
Search based approaches, instead, have to strongly limit the action space, keeping the number of motion primitives low, as the branching factor of the graph, i.e., the number of edges generated expanding each node, depends on the size of the database.

The effectiveness of \RRTMP\, has been validated in simulation, showing that the time required to build the tree is greatly reduced with the introduction of a LUT. This represents a promising result for online applications, especially in dynamic environments where the planner has to generate a new trajectory in response to changes in the obstacle-free state space.

An analysis of the probabilistic completeness and optimality properties of \RRTMP\,  is also provided. This involves a two-step procedure where we assess how close the proposed sample-based solution gets to the optimal one in the gridded state space as the number of samples grows to infinity, and how it gets close to the optimal solution in the continuous state space as the gridding gets finer and finer.

\subsection{Paper structure}
The paper is organized as follows. Section \ref{sec:Prob_state} introduces a formal description of the problem. In Section \ref{sec:algorithm} the proposed ${\tt {RRT^\star \text{\underline{\hspace{0.2cm}}} MotionPrimitives}}$ algorithm is explained in detail. An analysis of its probabilistic completeness and optimality properties is presented  in Section \ref{sec:analysis}. A numerical validation of \RRTMP\, is provided in Section \ref{sec:results}. Finally, some concluding remarks are drawn in Section \ref{sec:conc}.

\section{PROBLEM STATEMENT} \label{sec:Prob_state}
In this work, dynamical systems with state vector $\qvect \in \mathbb{R}^d$ and control input $\uvect \in \mathbb{R}^m$, governed by
\begin{equation} \label{eqn:system_dynamics}
\dot{\qvect}(t)=f \left(\qvect(t),\uvect(t)\right)
\end{equation}
where $f$ is continuously differentiable as a function of both arguments, are considered.
The control input $\uvect(t)$ is subject to  actuation constraints, and the admissible control space is denoted as $U \subset \mathbb{R}^m$. The state $\qvect(t)$ is constrained in the set $Q \subset \mathbb{R}^d$, and initialized with $\qvect(0)=\qvect_0\in Q$. Both $U$ and $Q$ are assumed to be compact. An open subset $Q_{goal}$ of $Q$ represents the goal region that the state has to reach.

A \emph{trajectory} of system \eqref{eqn:system_dynamics} is defined by the tuple $\mathbf{z}=\left(\qvect(\cdot), \uvect(\cdot),\tau\right)$, where $\tau$ is the duration of the trajectory, and $\qvect(\cdot):[0,\tau] \rightarrow Q $ and $\uvect(\cdot):[0,\tau] \rightarrow U$ define the state and control input evolution along the time interval $[0,\tau]$, satisfying the differential equation \eqref{eqn:system_dynamics} for $t \in [0,\tau]$,  the initial condition $\qvect(0)=\qvect_0$, and the final condition $\qvect(\tau)\in Q_{goal}$.

Obstacles are represented via an open subset $Q_{obs}$ of $Q$.
The free space is then defined as $Q_{free} := Q \setminus Q_{obs}$, and the assumption $\qvect_0 \in Q_{free}$ is enforced.\\
A trajectory $\mathbf{z}=\left(\qvect(\cdot), \uvect(\cdot),\tau\right)$ of system \eqref{eqn:system_dynamics} is  said to be \emph{collision free}, if it avoids collisions with obstacles, i.e., $\qvect(t) \in Q_{free}$, $t \in [0,\tau]$.\\
The set of collision free trajectories is denoted as $Z_{free}$.

An \emph{optimal kinodynamic motion planning problem}  can then be formalized as finding a feasible and collision free trajectory $\mathbf{z}^\star = \left(\qvect^\star(\cdot),\uvect^\star(\cdot),\tau^\star\right) \in Z_{free}$ that is \emph{optimal} according to a cost criterion $J(\mathbf{z}):Z_{free} \rightarrow \mathbb{R}_{\geq 0}$  that is expressed as
\begin{displaymath}
J(\mathbf{z}) = \int\limits_{0}^{\tau}g\left(\qvect(t),\uvect(t)\right)\mathrm{d}t
\end{displaymath}
where $g:Q\times U\rightarrow \mathbb{R}_{\geq 0}$ is an instantaneous cost function.  \BSnote{We assume that trajectories joining two different states have a non-zero cost. This is for instance the case in minimum time optimization where $g(\qvect, \uvect)=1$ and the trajectory duration $\tau$ is one of the optimization variables of the problem.}

\begin{remark}[translation invariance property]\label{rem:inv}
	In the context of motion planning, system \eqref{eqn:system_dynamics} represents the robot equations of motion and, consequently, its state vector $\qvect$ includes the robot position with respect to a given absolute reference frame.
	In the following, a translation invariance property is supposed to hold. This means that, if obstacles are neglected and a pair of initial and final states and the associated optimal trajectory $\mathbf{z}^\star = \left(\qvect^\star(\cdot),\uvect^\star(\cdot),\tau^\star\right)$ are considered, by shifting the origin of the coordinate system and rewriting all relevant quantities -- including system dynamics \eqref{eqn:system_dynamics}, and initial and final states -- in the new coordinate system and applying input $\uvect^\star(\cdot)$, the optimal robot trajectory is obtained, which is $\mathbf{z}^\star$ rewritten in the new coordinates.
\end{remark}

\section{RRT$^\star$ WITH MOTION PRIMITIVES} \label{sec:algorithm}
The approach here proposed is based on previous works on search-based \citep{pivtoraiko2009differentially,likhachev2009planning}, and sampling-based \citep{karaman2011sampling,karaman2010optimal} methods, and combines them in a novel way.\\
In particular, it relies on a uniform discretization of the state space, and on the computation of a finite set of motion primitives by solving a constrained optimization problem with boundary conditions on the  grid points of a smaller (uniform) grid. The motion primitives are then embedded in the RRT$^\star$ algorithm, where they are used to connect the randomly generated nodes to the tree, thus eliminating the need of solving online challenging and time consuming TPBVPs.

\subsection{Database of Motion Primitives} \label{sec:database_of_MotionPRIM}
The database of motion primitives is built by gridding the continuous state space in order to obtain a finite set of boundary conditions (initial and final states), and by solving offline a constrained boundary value optimization problem for each pair. The resulting set of optimal trajectories is then used repeatedly online, implementing a procedure that, when an edge connecting two nodes is requested by the planner, simply picks up from the database a suitable trajectory.

Given a state tuple $\qvect \in Q$, $\qvect = [\boldsymbol{\pi}^T, \dots]^T $, including the robot position $\boldsymbol{\pi}$, motion primitives are computed for each pair of initial and final states $(\bar{\qvect}_0,\bar \qvect_f)$ with $\bar{\qvect}_0=[\bar{\boldsymbol{\pi}}_0^T, \dots]^T$.
Given a boundary value pair $(\bar\qvect_0,\bar\qvect_f)$, a motion primitive is computed by solving the following optimization problem
\begin{align}
&\underset{\uvect(\cdot), \tau}{\text{mininimize}} && \int\limits_{0}^{\tau}g\left(\qvect(t),\uvect(t)\right)\mathrm{d}t\label{eqn:cost}\\
&\text{subject to }&&\dot{\qvect}(t)=f\left(\qvect(t),\uvect(t)\right) \nonumber\\
&&&\uvect(t)\in U, \, t \in [0,\tau] \nonumber\\
&&&\qvect(t)\in Q_{free}, \,  t \in [0,\tau] \nonumber \\
&&&\qvect(0)=\bar\qvect_0,\;\qvect(\tau)=\bar\qvect_f \nonumber
\end{align}

Finally, the database is generated by considering distinct pairs of initial and final states $(\bar\qvect_0^i,\bar\qvect_f^i)$, $i=1,\dots,N$, and computing the corresponding set of motion primitives  $\setLUT = \{\mathbf{z}^\star_i = \left(\qvect^\star_i(t),\uvect^\star_i(t),\tau^\star_i\right), \, i=1,\dots,N \} $ with the associated costs $\costLUT=\{C^\star_i, \, i=1,\dots,N\}$.

Note that, thanks to the translation invariance property introduced in Remark \ref{rem:inv}, the size of the database can be kept small while covering a wide range of the space where the robot is moving. Indeed, one can, without loss of generality, set the initial position $\bar{\boldsymbol{\pi}}_0$ to $\bar{\boldsymbol{\pi}}_0=0$ when building the database, and recover the optimal trajectory for an arbitrary initial position ${\boldsymbol{\pi}}_0$ by simply centring the coordinate system in ${\boldsymbol{\pi}}_0$.

\begin{figure}[hpbt]
	\centering
	\includegraphics[width=0.8\linewidth]{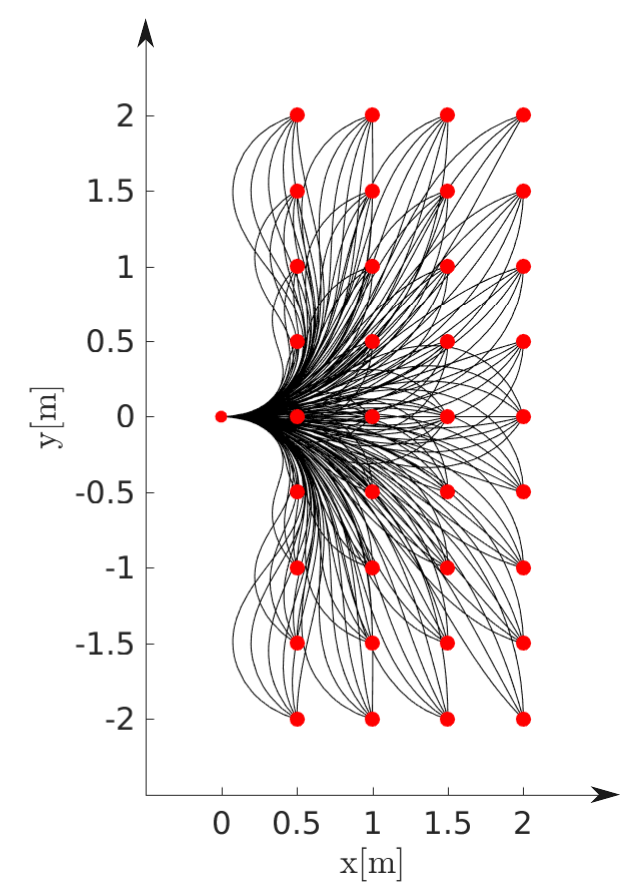}
	\caption{A subset of the motion primitives computed for a 3D search space $(x,y,\theta)$. Red dots correspond to the initial and final positions, and black lines represent the resulting trajectories for different final orientations ($\theta$). }
	\label{fig:subset_database}
\end{figure}

\begin{example} \label{ex:unicycle}
	Consider, as an example, a planning problem for a unicycle robot characterised by a 3D search space $(x,y,\theta)$, including the position $(x,y)$ and the orientation $\theta$, and by a 2D actuation space $(v,\omega)$, constituted by the linear velocity $v$ and the angular velocity $\omega$.\\
	Motion primitives are computed solving the following TPBVP
	\begin{align}
	&\underset{v(\cdot), \omega(\cdot), \tau}{\text{mininimize}} && \int\limits_{0}^{\tau} \left( 1 + 0.5v(t)^2 + 0.5\omega(t)^2 \right) \mathrm{d}t\nonumber\\[0.5em]
	&\text{subject to }&&\dot{x}(t)=v(t)\cos\left(\theta(t)\right) \nonumber\\
	&&&\dot{y}(t)=v(t)\sin\left(\theta(t)\right) \nonumber\\
	&&&\dot{\theta}(t)=\omega(t) \nonumber\\[0.5em]
	&&&v(t)\in \left[0,2\right], \, t \in [0,\tau] \nonumber\\
	&&&\omega(t)\in \left[-2,2\right], \, t \in [0,\tau] \nonumber\\[0.5em]
	&&&x(0)=\bar{x}_0,\; y(0)=\bar{y}_0,\; \theta(0)=\bar{\theta}_0 \nonumber\\
	&&&x(\tau)=\bar{x}_f,\; y(\tau)=\bar{y}_f,\; \theta(\tau)=\bar{\theta}_f \nonumber
	\end{align}
	for different initial and final poses.\\
	Figure \ref{fig:subset_database} shows a subset of these motion primitives, characterized by trajectories starting from $\bar{x}_0=\bar{y}_0=\bar{\theta}_0=0$.\\
	As can be seen in Figure \ref{fig:subset_database}, with the dynamical system and cost function considered in this example, motion primitives, corresponding to boundary conditions that are symmetric with respect to the $x$-axis, are symmetric. A further analysis reveals that the same property holds for the $y$-axis as well, and that symmetric primitives are characterized by the same cost.
\end{example}

\begin{figure}[hpbt]
	\centering
	\includegraphics[width=0.95\linewidth]{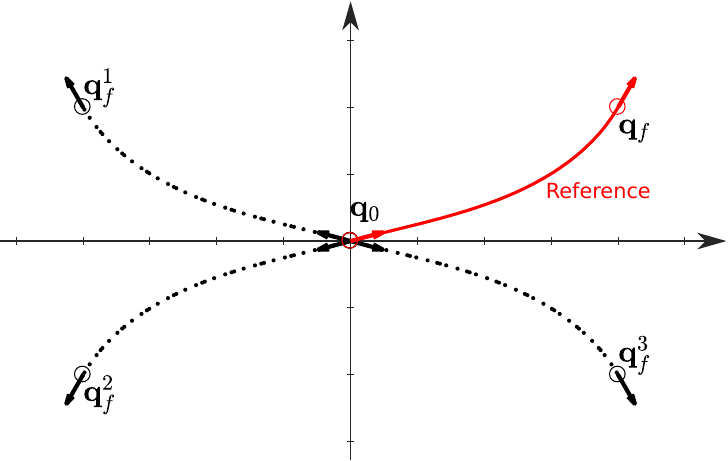}
	\caption{``Reference trajectory'' corresponding to $(\qvect_0,\qvect_f)$ (red solid line), and symmetric trajectories defined by the pair of boundary conditions $(\qvect_0,\qvect_f^1)$, $(\qvect_0,\qvect_f^2)$, $(\qvect_0,\qvect_f^3)$.}
	\label{fig:ref_traj}
\end{figure}

\begin{remark}
	When for the considered dynamical system and cost criterion stronger invariance properties hold, like the axis symmetry in Example \ref{ex:unicycle}, the size of the database can be further reduced by storing only a few ``reference trajectories'', and generating all the others using the invariance transformation.\\
	Figure \ref{fig:ref_traj} shows an example, referred to the TPBVP considered in Example \ref{ex:unicycle}, where the trajectories represented by the pairs $(\qvect_0,\qvect_f^1)$, $(\qvect_0,\qvect_f^2)$, $(\qvect_0,\qvect_f^3)$ are characterised by the same cost and can be easily mapped to a ``reference trajectory'' corresponding to the boundary value pair $\left(\qvect_0,\qvect_f\right)$. In this case, the size of the database can be further reduced storing only the ``reference trajectory''.\\
\end{remark}

\subsection{Search Space design}
In order to take advantage of the pre-computed database of motion primitives in sampling-based planning, the search space of the planner has to be defined appropriately so that every time the planner needs to connect two nodes, the corresponding optimal trajectory can be found in the database.
To guarantee that this is indeed the case, the search space of the planner is uniformly gridded as the region where motion primitives are built.
The translation invariance property\footnote{This approach can be easily extended in case stronger invariance properties hold.} can then be exploited, as any optimal trajectory which connects a pair of initial and final states (in the discretized search space) can be computed by first shifting the initial and final states so that the initial robot position corresponds to the zero position, then picking a suitable motion primitive in the database, and finally shifting the motion primitive so as to get back to the original reference coordinate system. Translation invariance, jointly with uniform gridding, allow  a reduced number of motion primitives to cover the entire (discretized) search space.
\BSnote{Note that, as the resolution of the database and the uniform grid size of the search space are coupled, we often use these two terms interchangeably.}

Determining optimal state space discretization depends on the specific application, and is out of the scope of this work. We shall assume here that the state space grid should include the initial state $\qvect_0$, at least a grid point in the goal region, and few points in the free space $Q_{free}$.
Moreover, in order to find a solution that reaches the goal region, the algorithm should be able to search within all homotopy classes \BSnote{that are feasible given the robot footprint}. In other words, one should be able to represent in the grid space  all sets of trajectories in the continuous state space that can be obtained applying a smooth transformation and lead to the goal region. Missing a homotopy class could highly deteriorate performance in terms of achieved cost.\\
In Section \ref{sec:results} some numerical examples are provided, in which different grids are used for solving the same case study and results are compared.

\begin{algorithm}[!t]
	$ { Q_T} \leftarrow \{ \qvect_0\}$, ${ E_T} \leftarrow \emptyset$, ${n} \leftarrow 1$\\	
	\While {$ n \leq N $} {
		$\qvect_\mathrm{rand} \leftarrow {\tt SAMPLE}(Q_\mathrm{free})$\\
		$Q_\mathrm{near} \leftarrow {\tt NEAR\_ NODES(\qvect_\mathrm{rand})}$\\
		$\qvect_\mathrm{best} \leftarrow {{\tt EXTEND}(Q_\mathrm{near},\qvect_\mathrm{rand})}$	\\
		\If{$\qvect_{rand} \not\in Q_T \land \qvect_\mathrm{best} \neq \emptyset $ }
		{$Q_T \leftarrow Q_T \cup \{ \qvect_\mathrm{rand}\}$\\
			$E_T \leftarrow E_T \cup \{\left(\qvect_\mathrm{best},\qvect_\mathrm{rand}\right)\}$\\
			$n \leftarrow n+1 $ \\
			$E_T \leftarrow {{\tt REWIRE}(Q_T, E_T, \qvect_\mathrm{rand},Q_\mathrm{near})}$}
		\ElseIf{\BSnote{$\qvect_{rand} \in Q_T$}}{
			\BSnote{$\qvect_{prev} \leftarrow {\tt PARENT}(\qvect_{rand}) $} \\
			\If {$\qvect_{best} \neq \qvect_{prev}$}{
				$E_T \leftarrow (E_T \setminus \{(\qvect_{prev}, \qvect_{rand})\}) \cup \{\qvect_{best},\qvect_{rand}\} $ \label{alg:extend_qrandInTree} \\
				$E_T \leftarrow {{\tt REWIRE}(Q_T, E_T, \qvect_\mathrm{rand},Q_\mathrm{near})}$
			}
		}					
	}
	\Return $Q_T,E_T$
	\caption{${\tt {RRT^\star \text{\underline{\hspace{0.2cm}}} MotionPrimitives}}$\label{algorithm:RRT*_motionprimitives}}
\end{algorithm}

\subsection{Motion Planning}
This section introduces \RRTMP\, (Algorithm \ref{algorithm:RRT*_motionprimitives}), the proposed variant of RRT$^\star$ integrating the database of motion primitives for the computation of a collision-free optimal trajectory (cf. Section \ref{sec:Prob_state}).

As in RRT$^\star$, \RRTMP\, is based on the construction of a random tree $T=(Q_T,E_T)$ where $Q_T\subset Q$ is the set of \emph{nodes}, and $E_T$ is the set of \emph{edges}. Nodes are states and edges are optimal trajectories, each one connecting a pair of origin and destination nodes and solving the TPBVP in \eqref{eqn:cost}.\\
The tree $T$ is expanded for a maximum number of iterations $N$, defined by the user, starting from $Q_T=\{\qvect_0\}$, where  $\qvect_0\in Q_{free}$ is the initial state, and $E_T=\emptyset$, as described in Algorithm \ref{algorithm:RRT*_motionprimitives}.\\
Every node $ \qvect\in Q_T$ is connected to $\qvect_0$ via a single sequence of intermediate nodes $\qvect_i \in Q_T$, $i=1, \dots, n-1$, $n\le N$, and associated edges  $e_i=e_{\qvect_i,\qvect_{i+1}} \in E_T$, $i=0,1, \dots, n-1$, with $\qvect_{n}=\qvect$.\\
One can then associate to this sequence a cost $C(\rightarrow \qvect_n)$ given by
\begin{displaymath}
C(\rightarrow \qvect_n) = \sum\limits_{i=0}^{n-1}C(e_i)
\end{displaymath}
where $C(e_i)$ denotes the cost associated with edge $e_i\in E_T$ and computed as in \eqref{eqn:cost}.

\begin{algorithm}[!t]
	$Q_\mathrm{near} \leftarrow \emptyset$ \\
	$\boldsymbol{\pi}_{rand}\leftarrow {\tt GetPosition}(\qvect_{rand}) $ \\
	\For{$\forall \qvect \in Q_T$}{			
		\If{$\qvect \in {\tt BoundingBox}(\boldsymbol{\pi}_{rand})$}{					
			\If{$C(e_{\qvect_{rand},\qvect}) \leq l(n) \lor C(e_{\qvect,\qvect_{rand}}) \leq l(n)$}{
				$Q_\mathrm{near} \leftarrow Q_\mathrm{near} \cup \{ \qvect\}$}
		}
	}
	\Return $Q_\mathrm{near}$
	\caption{${\tt {NEAR\_NODES}}$\label{algorithm:near_nodes}}
\end{algorithm}

Tree growing is based on four main steps -- random sampling, finding near nodes, extending the tree, and rewiring -- that are described in the following.

\subsubsection*{\textbf{Random sampling}}
A random state $\qvect_{rand}$ is sampled from the free state space $Q_{free}$ according to a uniform distribution by \BSnote{${\tt SAMPLE}\left( Q_\mathrm{free} \right)$}. Unlike the original RRT$^\star$ algorithm, however, the node is not sampled from the continuous state space, but from its discretization according to a uniform grid. For this reason, there is also a non zero probability that the same state $\qvect_{rand}$ is sampled again in the next iterations of the algorithm.

\subsubsection*{\textbf{Near nodes} (Algorithm \ref{algorithm:near_nodes})}
In RRT$^\star$ a random state $\qvect_{rand}$ can be connected only to a node that is within the set of its near nodes.\\
For Euclidean cost metrics, the set of near nodes is defined as a $d$-dimensional ball centered at $\qvect_{rand}$ of radius
\begin{displaymath}
\gamma_{ball} = \gamma_{RRT^\star}\left(\log(n)/n\right)^{1/d}
\end{displaymath}
where $n$ is the tree cardinality at the current iteration of the algorithm and $\gamma_{RRT^\star}$ is a suitable constant selected as
\begin{displaymath}
\gamma_{RRT^\star} > 2 \left(1+1/d\right)^{1/d}\left(\mu(Q_{free})/\zeta_d\right)^{1/d}
\end{displaymath}
$\mu(Q_{free})$ and $\zeta_d$ denoting the volume of the free configuration space and of the unit ball, respectively, in a $d$-dimensional Euclidean space.\\
For non-Euclidean cost metrics, the distance between two states is represented by the optimal cost of the trajectory that connects them, and near nodes are selected from a set of reachable states, $Q_{reach}$, defined as the set of states that can be reached from $\qvect_{rand}$ or that can reach $\qvect_{rand}$ with a cost that satisfies some threshold value. More specifically,
\begin{equation}
\begin{split}
Q_{reach}= \{\qvect \in Q : C(e_{\qvect_{rand},\qvect}) \leq l(n) \lor \\
C(e_{\qvect,\qvect_{rand}}) \leq l(n)\}
\end{split}
\label{eqn:Qreachdef}
\end{equation}
where $e_{\qvect_i,\qvect_j}$ denotes the edge from $\qvect_i$ to $\qvect_j$, and $l(n)$ is a cost threshold that decreases over the iterations of the algorithm as $l(n)=\gamma_l \left(\log n/n\right)$ such that a ball of volume $\gamma^{d}\left(\log n/n\right)$ is contained within $Q_{reach}$, where $\gamma_l$ and $\gamma$ are suitable constants \citep[see][for further details]{karaman2010optimal}.
\begin{figure*}[htbp]
	\centering
	\subfigure[original query]{\label{fig:roto-trans-a}\includegraphics[width=.24\linewidth]{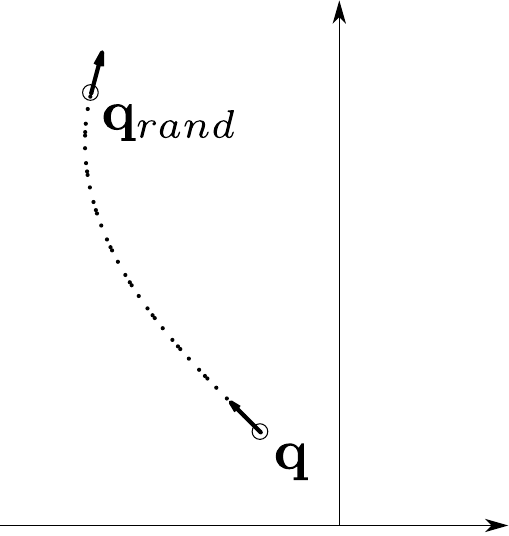}} \hspace{0.4cm}
	\subfigure[translation to origin]{\label{fig:roto-trans-b}\includegraphics[width=.24\linewidth]{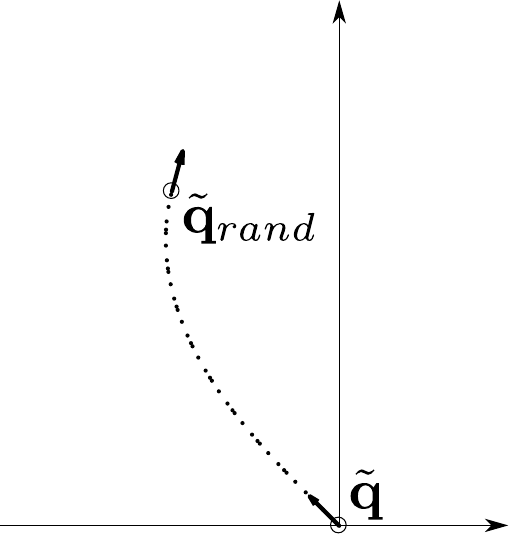}} \hspace{0.4cm}
	\subfigure[inverse transformation and edge design]{\label{fig:roto-trans-d}\includegraphics[width=.24\linewidth]{rot1.pdf}}
	\caption{Steps involved in the coordinate transformation within the ${\tt FindTrajectory}$ procedure.\label{fig:roto-trans}}
\end{figure*}

In \RRTMP, however, the set $Q_{reach}$ has to be further constrained to ensure that there is a pre-computed trajectory in the database for each connection in the set of near nodes, that is thus redefined as follows
\begin{equation*}
Q_{near}=Q_{reach} \cap {\tt BoundingBox}(\boldsymbol{\pi}_{rand})
\end{equation*}
where ${\tt BoundingBox}(\boldsymbol{\pi}_{rand})$ denotes the box of grid points in the state space $Q$ adopted for the database construction, with the origin of the reference coordinate system shifted from $\bar \pi_0=0$ to  $\pi_{rand}$, which is the robot position associated to state $\qvect_{rand}$ and obtained by using ${\tt GetPosition}$.\\
If $Q_{near}$ occurs to be an empty set, the algorithm continues to the next iteration selecting a new $\qvect_{rand}$.

\subsubsection*{\textbf{Extend} (Algorithm \ref{algorithm:extend})}
The tree is extended to include $\qvect_{rand}$ by selecting the node $\qvect_{best}\in Q_T$ such that the edge $e_{\qvect_{best},\qvect_{rand}}$ connects $\qvect_{rand}$ with a minimum cost collision free trajectory.\\
$\qvect_{best}$ is determined as follows
\begin{equation*}
\qvect_{best} = \argmin\limits_{\qvect \in Q_{feasible}} C(\rightarrow \qvect) + C(e_{\qvect,\qvect_{rand}})
\end{equation*}
where $Q_{feasible}\subseteq Q_{near}$ is the set of nodes $\qvect$ that belong to $Q_{near}$ and such that the trajectory connecting $\qvect$ to $\qvect_{rand}$ is collision free, i.e.,
\begin{equation*}
Q_{feasible} = \{\qvect \in Q_{near} | {\tt CollisionFree} (e_{\qvect,\qvect_{rand}})=1 \}
\end{equation*}
where ${\tt CollisionFree}: E_T \to \{0,1\}$ is a function that returns 1 for an edge that is collision free, 0 otherwise.\\
The selection of an edge from the database, connecting $\qvect$ to $\qvect_{rand}$, is a peculiarity of \RRTMP, and is performed by the ${\tt FindTrajectory}$ function as follows:
\begin{enumerate}
	\item a translation is applied to the pair of initial and final states $\left(\qvect,\qvect_{rand}\right)$ (Figure \ref{fig:roto-trans-a}), obtaining the normalized pair $\left(\tilde{\qvect},\tilde{\qvect}_{rand}\right)$, such that the resulting $\tilde{\qvect}$ has the position $\tilde{\boldsymbol{\pi}}$ corresponding to the null vector (Figure \ref{fig:roto-trans-b});
	\item  a query is executed on the database to look for the trajectory $\mathbf{z}_i \in \setLUT$ and the cost $C(e_{\tilde{\qvect},\tilde{\qvect}_{rand}}) \in \mathbb{R}$;
	\item the inverse of the previous translation is applied in order to recover the trajectory connecting the actual pair of boundary values $\left(\qvect,\qvect_{rand}\right)$, determining the edge $e_{\qvect,\qvect_{rand}}$ (Figure \ref{fig:roto-trans-d}).
\end{enumerate}
At the end of this procedure, if $\qvect_{rand}$ is not already in the tree, then it is added to the tree together with the minimum cost edge, i.e., $Q_T$ is replaced by $\{\qvect_{rand} \}\cup Q_T$ and $E_T$ by $\{e_{\qvect_{best},\qvect_{rand}}\}\cup E_T$ (see steps 7 and 8 in Algorithm \ref{algorithm:RRT*_motionprimitives}). If $\qvect_{rand}$ is already in the tree and if the computed $\qvect_{best}$ is different from the current parent node, $\qvect_{prev}$, of $\qvect_{rand}$, \BSnote{given by ${\tt PARENT}(\qvect_{rand})$}, then the previous edge, $e_{\qvect_{prev},\qvect_{rand}}$, is replaced by the new edge,  $e_{\qvect_{best},\qvect_{rand}}$ (see step \ref{alg:extend_qrandInTree} in Algorithm \ref{algorithm:RRT*_motionprimitives}).

\begin{algorithm}[!t]
	$\qvect_\mathrm{best}\leftarrow \emptyset, e_\mathrm{best} \leftarrow \emptyset, \tt c_\mathrm{best} \leftarrow \infty $ \\
	\For {$ \qvect \in  Q_\mathrm{near} $} {
		$ z, {C(e_{\qvect, \qvect_\mathrm{rand} })} \leftarrow {\tt FindTrajectory}\left(\qvect, \qvect_\mathrm{rand} \right)$\\		
		\If{${C}(e_{\qvect, \qvect_\mathrm{rand} }) < c_\mathrm{best}  $ }{
			\If{$\tt CollisionFree(e_{\qvect, \qvect_\mathrm{rand} })$}{
				$ \qvect_\mathrm{best} \leftarrow \qvect$\\
				$e_\mathrm{best} \leftarrow z$\\
				$ c_\mathrm{best} \leftarrow {C}(e_{\qvect, \qvect_\mathrm{rand} }) $			
			}		}	
		}
		\Return $\qvect_\mathrm{best}$
		\caption{${\tt {EXTEND}}$\label{algorithm:extend}}
	\end{algorithm}

	\begin{algorithm}
		\For {$ \qvect \in Q_\mathrm{near} $} {
			$ z, {C(e_{\qvect_\mathrm{rand},\qvect})} \leftarrow {\tt FindTrajectory}\left(\qvect_\mathrm{rand},\qvect\right)$\\		
			\If{$C \left(\rightarrow \qvect_{rand}\right) + {C}(e_{\qvect_\mathrm{rand},\qvect}) < C \left(\rightarrow \qvect\right) $ }{
				\If{$\tt CollisionFree(e_{\qvect, \qvect_\mathrm{rand} })$}{
					$\qvect_{parent} \leftarrow {\tt Parent}\left(\qvect\right)$ \\
					$e_{prev} \leftarrow e_{\qvect_{parent},\qvect} $ \\
					$e \leftarrow e_{\qvect_{rand},\qvect}$ \\
					$E_T= \{E_T \setminus e_{prev}\} \cup \{e\}$}	
			}			
		}
		\Return $E_T$
		\caption{${\tt {REWIRE}}$\label{algorithm:rewire}}
	\end{algorithm}

\subsubsection*{\textbf{Rewiring} (Algorithm \ref{algorithm:rewire})}
In order to ensure that \BSnote{all node pairs are connected by an optimal sequence of edges}, every time a new node $\qvect_{rand}$ is added to the tree, a check is performed to verify if an already existing node can be reached from this newly added node with a smaller cost.\\
Therefore, $\forall \qvect \in Q_{near}$ if $e_{\qvect_{rand},\qvect}$ is collision free, and the following conditions hold
\begin{align*}
&C(e_{\qvect_{rand},\qvect}) \leq l(n)\\
&C(\rightarrow \qvect_{rand}) + C(e_{\qvect_{rand},\qvect}) \leq C(\rightarrow \qvect)
\end{align*}
the tree is rewired, i.e.,
\begin{displaymath}
E_T \leftarrow \{E_T \setminus e_{prev}\} \cup \{e_{\qvect_{rand},\qvect}\}
\end{displaymath}
where $e_{prev}$ is the previous edge connecting the node $\qvect$ to the tree.

\subsubsection*{\textbf{Termination and best sequence selection}}
After the maximum number of iterations is reached the procedure to build the tree terminates.\\
The best trajectory is selected as the node sequence reaching the goal region with the minimum cumulative cost $C$.\\
Note that, using a discretized search space limits the number of nodes that can be sampled, once all of them have been sampled the tree cardinality does not increase any more, but the algorithm can still continue updating the edges to ensure that each node is connected with the best possible parent node.



\section{COMPLETENESS AND OPTIMALITY ANALYSIS}\label{sec:analysis}
In this section,  probabilistic completeness of the proposed planning algorithm and optimality of the solution are discussed. Furthermore, some results to assess how close the solution obtained using a discretized state space and motion primitives is to the optimal trajectory computed considering a continuous state space are provided.

Let $Q^{\Delta}$ define the set of grid points that represent the discretized state space\footnote{In this section, a $\Delta$ superscript is used to denote all variables that are associated with the grid state space, so as to distinguish them from their continuous state space counterpart.} and similarly $Q^{\Delta}_{free} := Q^\Delta \cap Q_{free} $ represents the free discrete state space. Assuming that the discretization step size is chosen properly, then, the collection of all grid points $\qvect \in Q^{\Delta}_{free}$ that can be reached from $\qvect_0$ by concatenating a sequence of motion primitives in $Q_{free}$ is a non empty set. We shall denote this set as $V^\Delta_{free}$ and its cardinality as $N^\Delta$. Note that the end points of the concatenated motion primitives are grid points in $ Q^{\Delta}_{free}$ and hence they belong to  $V^\Delta_{free}$.

Let $\mathcal{G}^\Delta_{free} = (V^{\Delta}_{free}, E^\Delta_{free})$ be a graph where the set of nodes is given by $V^\Delta_{free}$ defined before and the set of edges $E^\Delta_{free}$ is the collection of all the (possibly translated) motion primitives iteratively built as follows: starting from $\qvect_{0}$ consider all the (translated) motion primitives that lie in $Q_{free}$ and connect $\qvect_0$ to all possible grid points in $Q^\Delta_{free}$, and, then, continue with the same strategy for all of the newly reached grid points iteratively until it is not possible to further expand the graph.

Finally, $Q_{goal}^{\Delta}$ denotes the set of those grid points of $V^{\Delta}_{free}$ that belong to $Q_{goal}$. The motion planning problem using the grid representation admits a solution if $Q_{goal}^{\Delta}$ is not empty since this means that there exists a way of reaching a state in $Q_{goal}$ starting from $\qvect_{0}$ with the available motion primitives. In the following derivations we assume that $Q_{goal}^{\Delta} \neq \emptyset$.

Similarly to RRT$^\star$, \RRTMP\, generates a tree $T$ based on the random samples extracted from $Q^\Delta_{free}$. However, unlike RRT$^\star$, the nodes and edges added to the tree belong respectively to $V^\Delta_{free}$ and $E^\Delta_{free}$, so that the obtained tree $T$ is a sub-graph of $\mathcal{G}^\Delta_{free}$, i.e., $T \subset \mathcal{G}^\Delta_{free}$. The subscript $i$ is used to denote the generated tree and the cost of the lowest cost trajectory represented in that tree after $i$-th iterations, i.e., $T_i$ and $c_i$, respectively.

Since $Q_{goal}^{\Delta} \neq \emptyset$, then, there exists at least an \emph{optimal branch} of $\mathcal{G}^\Delta_{free}$ composed of the ordered sequence of nodes
\begin{equation*}
S^\star := \{\qvect_0^\star,\qvect_1^\star,\dots,\qvect_k^\star\},
\end{equation*}
which represents a \emph{resolution optimal $\Delta$-trajectory}, a minimum cost trajectory that starts at the initial state $\qvect_0$ and ends in the goal region, i.e., $\qvect_k^\star \in Q_{goal}^{\Delta}$, such that  $\qvect_j^\star \in V^{\Delta}_{free}$ and $e_{\qvect_{j-1}^\star,\qvect_{j}^\star} \in E^\Delta_{free}$, for any $j=1, \dots,k$. Let $c^{\star\Delta}$ denote the cost of this optimal trajectory, i.e., $C(\rightarrow \qvect_k^\star)= c^{\star\Delta}$, which is named \emph{resolution optimal $\Delta$-cost}.

The goal of \RRTMP\, can then be reformulated as that of generating a tree that contains an optimal branch $S^\star$ to reach $Q^{\Delta}_{goal}$ that is represented in $\mathcal{G}^\Delta_{free}$. Note that, one could in principle build $G^\Delta_{free}$ and apply an exhaustive search on it. However, this can still be an issue due to the combinatorial nature of the problem, in particular due to the branching caused by the dimensionality of the state space and the number of nodes contained in the graph.

In this section the quality of the solution obtained by \RRTMP\, is analyzed by addressing the following questions:
\begin{enumerate}
	\item resolution optimality: if there exists a resolution optimal $\Delta$-trajectory $S^\star$ in $\mathcal{G}^\Delta_{free}$, then, is it possible to obtain such a trajectory?
	\item asymptotic optimality: how close is the resolution optimal $\Delta$-cost to the cost of the optimal trajectory, as the grid resolution increases and the grid converges to the continuous state space?
\end{enumerate}

\begin{theorem}	\label{thm:resOpt}
	As the number of iterations goes to infinity the cost of the trajectory returned by \RRTMP\, converges to the {resolution optimal $\Delta$-cost} with a probability equal to 1, i.e.,
	\begin{equation*}
	\mathbb{P}\left( \left\{\lim_{i \rightarrow \infty} c_i = c^{\star\Delta} \right\} \right)=1.
	\end{equation*}
	Moreover, the expected number of iterations required to converge to $c^{\star\Delta}$ is upper bounded by $k|Q_{free}^{\Delta}|$, where $k$ is the length of an optimal branch of $\mathcal{G}^\Delta_{free}$.
\end{theorem}

\begin{proof}
	\RRTMP\, returns the resolution optimal solution if it discovers an optimal branch, $S^\star = \{\qvect_0^\star,\qvect_1^\star,\dots,\qvect_k^\star\}$, defined on $\mathcal{G}_{free}^\Delta$. Assuming that the constant defining the cost threshold for selecting nearby nodes, $\gamma_l$, is selected large enough (e.g. $\gamma_l$ can be chosen so that $Q_{reach}$ defined in \eqref{eqn:Qreachdef} contains the ${\tt BoundingBox}$\
	$(\boldsymbol{\pi}_{rand})$ for $n=N^\Delta$), as soon as $\{\qvect_0^\star,\qvect_1^\star,\dots,\qvect^\star_{j-1}\}$ is a branch in the tree, $T=(Q_T,E_T)$, then the algorithm adds the edge $e_{\qvect_{j-1},\qvect_j}$ to the tree in one of the following two ways:
	\begin{enumerate}
		\item any time $\qvect_{j}^\star$ is sampled it is connected to $\qvect_{j-1}^\star$ by the {\tt EXTEND} procedure. In fact, as $e_{\qvect_{j-1}^\star,\qvect_j^\star}$ is part of the optimal sequence, there is no better way of reaching $\qvect_{j}^\star$ other than $e_{\qvect_{j-1}^\star,\qvect_j^\star}$.
		
		\item if $\qvect_{j}^\star$ is sampled before $\qvect_{j-1}^\star$ and connected to some node $\qvect_{prev}$, such that $\qvect_j^\star \in Q_T$ and $e_{\qvect_{prev},\qvect_j^\star} \in E_T$, the edge $e_{\qvect_{j-1}^\star,\qvect_j^\star}$ is selected any time $\qvect_{j-1}^\star$ is sampled, thanks to the {\tt REWIRE} procedure. 
	\end{enumerate}
	This property allows to model the process of determining the sequence $S^\star$ as an \emph{absorbing Markov chain} initialized at $\qvect_0^\star$ with $\qvect_k^\star$ as absorbing state and all intermediate $\qvect_j^\star$, $j=1,2, \dots k-1 $ that are transient states. There is a positive probability $\mathcal{P}_j$ of advancing in the sequence, i.e., moving from $\qvect^\star_{j-1}$ to $\qvect^\star_j$, and a probability $1-\mathcal{P}_j$ of staying at the same state. Considering that at each iteration a new grid point is sampled from $Q_{free}^{\Delta}$ independently and according to a uniform distribution, each one has a probability $\frac{1}{|Q_{free}^{\Delta}|}$ of being extracted. Therefore, there is a probability $\mathcal{P}_j=\frac{1}{|Q_{free}^{\Delta}|}$ of advancing in the \emph{Markov chain} and, since all states are transient states apart from $\qvect_k^\star$ which is the absorbing state, the probability that the process is absorbed by $\qvect_k^\star$ tends to 1 as $i$ tends to infinity \citep{bertsekas2002introduction}. Moreover, there is a finite number of expected iterations, $k/\mathcal{P}_j=k\,|Q_{free}^{\Delta}|$, before the process is absorbed, i.e., before \RRTMP\, returns the resolution optimal solution.
\end{proof}

Clearly, the number of expected iterations increases with the depth of the solution, $k$, and the number of states represented in the grid. However, one advantage of \RRTMP, as RRT$^\star$ is the possibility of obtaining a solution rapidly and possibly improving its quality within the allowed computing time.

\begin{corollary}
	\RRTMP\, is probabilistically resolution complete, as the number of iterations goes to infinity the algorithm will return a solution to the motion planning problem, if there exists one in $\mathcal{G}_{free}^\Delta$, with a probability 1.
\end{corollary}

The remaining of this section deals with the relation between the resolution optimal $\Delta$-trajectory returned by \RRTMP\, as the number of iterations grows to infinity and the truly optimal trajectory in the continuous state space. To this purpose, we shall focus with the case when there are no actuation constraints and enforce the following assumptions regarding the properties of the dynamical system \eqref{eqn:system_dynamics} and the existence of a solution.

\begin{assumption} The following properties hold for the dynamical system in \eqref{eqn:system_dynamics}
	\begin{itemize}
		\item the system is small-time locally attainable (STLA)\footnote{A system is STLA from a state $\qvect \in Q$ if $\forall T>0$ the reachable set of states from $q$ in time $0<t \leq T$, $\mathcal{R}(\qvect, \leq T)$ contains a d-dimensional subset of $\mathcal{N}$, where $\mathcal{N}$ denotes the set of neighborhood states in terms of Euclidean distance,  \citep{choset2005principles}.}; 	
		\item function $f(\cdot)$, representing the system dynamics, is Lipschitz continuous with Lipschitz constant $\mathcal{K}_f$;
		\item function $C(\cdot)$, assigning a cost to an edge, \BSnote{satisfies} the following Lipschitz-like continuity condition with Lipschitz constant $\mathcal{K}_c$:
		\begin{displaymath}
		\left| C(e_{\qvect_0,\qvect_1})- C(e_{\tilde{\qvect}_0, \tilde{\qvect}_1})\right| \leq \mathcal{K}_c \left\| \left| \begin{array}{c}
		\qvect_0 \\ \qvect_1 \\ \end{array} \right| - \left| \begin{array}{c} \tilde{\qvect}_0 \\ \tilde{\qvect}_1 \\ \end{array} \right| \right\|
		\end{displaymath}
		for each pair of edges $e_{\qvect_0,\qvect_1}$ and $e_{\tilde{\qvect}_0, \tilde{\qvect}_1}$.
	\end{itemize}
	\label{assum:dyn_syst}
\end{assumption}

\BSnote{Derivations in the rest of this section apply straightforwardly to state spaces that are Euclidean, and can be generalized to state spaces that are manifolds if the following assumption holds. }
\begin{assumption} The state space manifold of \BSnote{system \eqref{eqn:system_dynamics} with $d$ state variables} is a subspace of the $d$-dimensional Euclidean space, $\mathbb{R}^d$, therefore can be locally treated as $\mathbb{R}^d$. \label{assum:state_space}
\end{assumption}

With a slight abuse of the previously introduced notation, in the rest of this section we use the term ``trajectory'' for the state space component of the tuple $\mathbf{z}$ defined in Section \ref{sec:Prob_state}.
In order to compare the trajectory returned by \RRTMP\, and the optimal trajectory in the continuous state space, firstly, trajectories whose points are all away from obstacles by a certain distance are considered. For this reason, the definition of obstacle clearance of a trajectory, i.e., the minimum distance between obstacles and points belonging to the trajectory, has to be introduced.

\begin{definition}[$\epsilon$-obstacle clearance]
	Given a trajectory $\sigma(t)$, $t \in [0,T]$, if the ball $\mathcal{B}_{\epsilon}\left(\sigma(t)\right)$ of radius $\epsilon$ and centred at $\sigma(t)$ is strictly inside $Q_{free}$, for any $t \in [0,T]$, then, the $\epsilon$-obstacle clearance property holds for $\sigma$.
\end{definition}

\begin{definition}[$\epsilon$-free space]
	Let $\sigma(t):[0,T] \rightarrow Q_{free}$ be a trajectory which has $\epsilon$-obstacle clearance, the $\epsilon$-free space along $\sigma$ is given by
	\begin{equation*}
	Q^{\epsilon}_{\sigma} := \bigcup_{t \in [0,T]}\mathcal{B}_{\epsilon}\left(\sigma(t)\right)
	\end{equation*}
\end{definition}

\noindent Then, a set of trajectories that are called to be $\epsilon$-similar to $\sigma$ can be introduced. 

\begin{definition}[$\epsilon$-similarity]
	Any trajectory $\tilde{\sigma}(t):[0,\tilde{T}] \rightarrow Q_{free}$ is said to be $\epsilon$-similar to $\sigma$ if it lies in the $Q^{\epsilon}_{\sigma}$ free-space, i.e., if
	\begin{displaymath}
	\tilde{\sigma}(t) \in Q^{\epsilon}_{\sigma}, \, t \in [0,\tilde{T}].
	\end{displaymath}
\end{definition}
\begin{figure}[t!]
	\centering
	\includegraphics[width=0.75\linewidth]{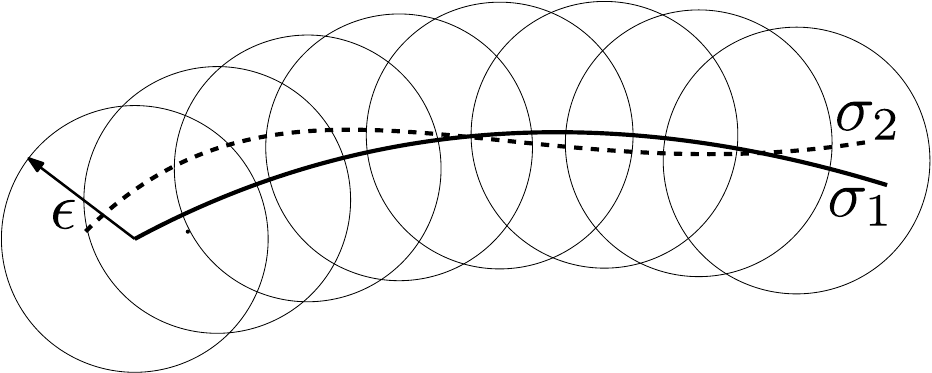}
	\caption{An example of a trajectory $\sigma_2$ that is $\epsilon$-similar to a trajectory $\sigma_1$.}
	\label{fig:epsSimilar}
\end{figure}

\noindent Figure~\ref{fig:epsSimilar} shows an example of a trajectory that is $\epsilon$-similar to another one.

Note that having an $\epsilon$-free space along a trajectory is not sufficient to guarantee the existence of $\epsilon$-similar trajectories ($\epsilon$-dynamic clearance), as this existence depends also on the properties of the dynamical system \eqref{eqn:system_dynamics}. The following definition relates the $\epsilon$-similarity with the $\epsilon$-free space through the system dynamics.

\begin{definition}[$\epsilon$-dynamic clearance]
	\BSnote{Given a trajectory $\sigma(t) \colon [0,T] \rightarrow Q_{free}$ which has $\epsilon$ obstacle clearance, if for any pair of time instants $ t_1, t_2$, such that $0 \leq t_1 < t_2 \leq T$, there exists a set of states inside a ball of radius $\alpha \epsilon$, with $0 < \alpha \leq 1$, centered at $\sigma(t_2)$ that are reachable from $\sigma(t_1)$ according to dynamics \eqref{eqn:system_dynamics} without leaving the $\epsilon$-free space around $\sigma(t)$, then $\sigma$ has $\epsilon$-dynamic clearance.}
\end{definition}

\begin{figure}[t!]
	\centering
	\includegraphics[width=0.75\linewidth]{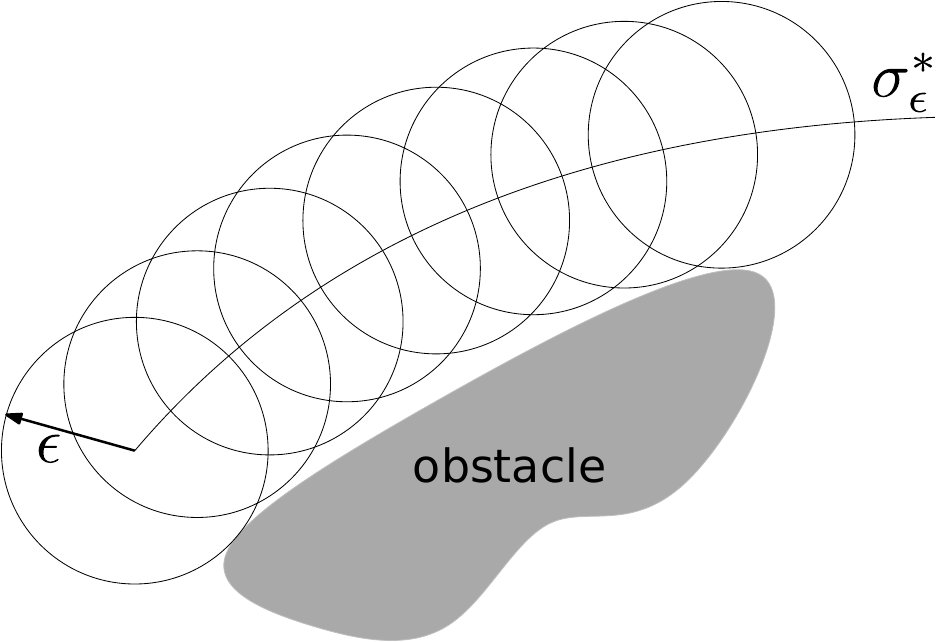}
	\caption{$\sigma_\epsilon^\star$ is the optimal trajectory with at least $\epsilon$-obstacle and dynamic clearance.}
	\label{fig:covering_balls}
\end{figure}

Let $\Sigma_\epsilon$ denote all the trajectories that solve the motion planning problem and have at least $\epsilon$-obstacle and dynamic clearance. Let $c^\star_\epsilon$ denote the minimum cost over all $\Sigma_\epsilon$, which corresponds to the \BSnote{$\epsilon$\textit{-optimal trajectory}}, $\sigma^\star_\epsilon$. Note that, as $\epsilon$ tends to zero $\sigma^\star_\epsilon$ converges to the truly optimal trajectory in the continuous state space. Due to the discretized nature of \RRTMP\, it is not possible to converge to the optimal trajectory in the continuous state space, however in the following it is proven that \BSnote{the graph, $\mathcal{G}^\Delta_{free}$, contains an $\epsilon$-similar trajectory, $\sigma^{\Delta}_\epsilon$, to the optimal trajectory $\sigma^\star_\epsilon$ with a particular $\epsilon$-clearance.} Consequently, this result will be used to show that as the resolution of the grid increases, and as $\epsilon$ converges to zero, the \emph{resolution optimal $\Delta$-cost} will converge to the truly optimal cost, i.e., $c^\star$.

\begin{theorem}
	\BSnote{Let $\bar{\epsilon} >0$ be smaller than half of the shortest side of the bounding box and such that the system \eqref{eqn:system_dynamics} admits an $\epsilon$-optimal trajectory for any $\epsilon \leq \bar{\epsilon}$, then,}
	\begin{itemize}
		\item for a sufficiently fine gridding, $\mathcal{G}^\Delta_{free}$ contains an $\epsilon$-similar trajectory to $\sigma^\star_\epsilon$, \BSnote{$\forall \epsilon \leq \bar{\epsilon}$},
		\item as the grid resolution increases, the \emph{resolution optimal $\Delta$-cost} converges to $c^\star$.
	\end{itemize}
	\label{thm:AsmpOpt}
\end{theorem}

\begin{figure}[t!]
	\centering
	\includegraphics[width=0.88\linewidth]{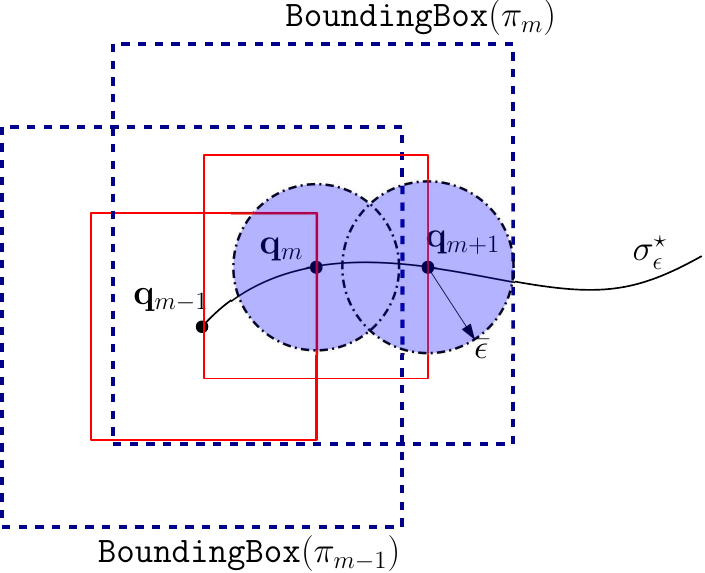}
	\caption{Consecutive samples taken along the $\epsilon$-optimal trajectory to design a $\{\mathcal{B}_{\delta}\}$ sequence. The corresponding ${\tt BB}(\boldsymbol{\pi}_m)$ are shown in red.}
	\label{fig:bbox seqn}
\end{figure}

\begin{proof}
\BSnote{Let us fix $\epsilon$, with $0 < \epsilon \leq \bar{\epsilon}$, and sample a set of states, $\{\qvect_m \colon m=0,1, \dots, M \}$, along the $\epsilon$-optimal trajectory, $\sigma_\epsilon^\star$, starting from the initial state and ending in the final state, in such a way that a ball of radius $\bar{\epsilon}$ centered at sample $\qvect_{m}$ would be contained in the {\tt BoundingBox}$(\boldsymbol{\pi}_{m-1})$ centered at the preceding sample $\qvect_{m-1}$ and touching its boundary for $ m=1, \dots, M-1$ (Figure~\ref{fig:bbox seqn}), till the final state is included within {\tt BoundingBox}$(\boldsymbol{\pi}_{M-1})$.}

\BSnote{As $\sigma_\epsilon^\star$ has $\epsilon$-dynamic clearance, it is true that for any sample $\qvect_{m}$ along the trajectory there exists a set of states within a ball of radius $\alpha\epsilon$ centered at $\qvect_{m+1}$ that is reachable without leaving $Q^\epsilon_{\sigma_\epsilon^\star}$, which is the $\epsilon$-free space along $\sigma_\epsilon^\star$. From Assumption \ref{assum:dyn_syst}, in particular from the Lipschitz continuity of the system dynamics, it follows that a sequence of non-overlapping balls, $\{\mathcal{B}_{\delta}\left(\qvect_m\right)\}$, centered at the samples along $\sigma_\epsilon^\star$ and characterized by radius $\delta$ where
\begin{equation}
\delta = \frac{\alpha \epsilon}{2 \mathcal{K}_f}
\label{eqn:deltaDefinition}
\end{equation}
can be determined such that any state within $\mathcal{B}_{\delta}\left(\qvect_{m}\right)$ can reach any state in $\mathcal{B}_{\delta}\left(\qvect_{m+1}\right)$ without leaving\footnote{It is assumed that $\mathcal{K}_f \geq 1$.} $Q^\epsilon_{\sigma_\epsilon^\star}$ \cite[see][]{khalil1996nonlinear,karaman2010optimal}.}
\BSnote{Assuming that the discretization is fine enough so that none of the $\mathcal{B}_{\delta}\left(\qvect_m\right)$ is empty, there exists a sequence of nodes $ \{ \tilde{\qvect}_m \colon m=0,1, \dots, M \} $ represented in $\mathcal{G}^\Delta_{free}$, such that $ \tilde{\qvect}_m \in \mathcal{B}_{\delta}\left(\qvect_{m}\right), m=0,1, \dots, M $, $\tilde{\qvect}_0 = \qvect_0$, $\tilde{\qvect}_M$ belongs to the goal region and the corresponding trajectory $\sigma_\epsilon^\Delta$ in $\mathcal{G}^\Delta_{free}$ is $\epsilon$-similar to $\sigma_\epsilon^*$ (Figure~\ref{fig:eps_balls}).}

\begin{figure}[t!]
		\centering
		\includegraphics[width=0.9\linewidth]{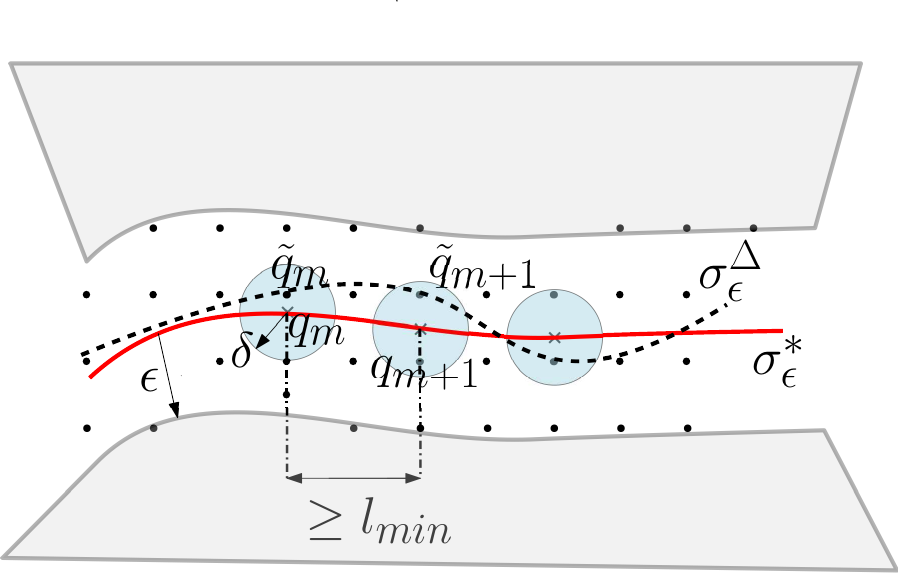}
		\caption{Construction of the ball sequence $\{\mathcal{B}_{\delta}\left(\qvect_m\right)\}$. Tiling $\sigma_\epsilon^\star$ with balls of radius $\delta$ centered at consecutive samples, there exist trajectories that are $\epsilon$-similar to $\sigma_\epsilon^\star$ represented on $\mathcal{G}^\Delta_{free}$. The cost of the trajectory connecting two consecutive samples is larger than or equal to $l_{min}$.}
		\label{fig:eps_balls}
	\end{figure}
	
\BSnote{The cost of the optimal trajectory can be defined as the sum of the costs of the trajectories connecting consecutive $\mathcal{B}_{\delta}\left(\qvect_m\right)$ balls, i.e.,
	\begin{equation}
	c_{\epsilon}^* = \sum_{m=1}^{M} C(e_{\qvect_{m-1} , \qvect_{m}}),
	\label{eqn:cost_segments}
	\end{equation}
	and similarly the cost of the $\sigma^{\Delta}_\epsilon$ is
	\begin{equation}
	c^{\Delta}_{\epsilon} = \sum_{m=1}^{M} C(e_{\tilde{\qvect}_{m-1}, \tilde{\qvect}_{m}}).
	\label{eqn:cost_segmentsEsimilar}
	\end{equation}
	Subtracting \eqref{eqn:cost_segments} from \eqref{eqn:cost_segmentsEsimilar}, the difference can be written as
	\begin{equation*}
	c^{\Delta}_{\epsilon} - c_{\epsilon}^*= \sum_{m=1}^{M} \left( C(e_{\tilde{\qvect}_{m-1}, \tilde{\qvect}_{m}}) - C(e_{\qvect_{m-1}, \qvect_{m}})\right).
	\end{equation*}
	By the Lipschitz continuity of the cost function given in Assumption~\ref{assum:dyn_syst} it follows that
	\begin{equation*}
	c^{\Delta}_{\epsilon} - c_{\epsilon}^* \leq \sum_{m=1}^{M} \mathcal{K}_c \delta,
	\end{equation*}
	where $\mathcal{K}_c$ is the Lipschitz constant for the cost function. Therefore,
	\begin{equation}
	c^{\Delta}_{\epsilon} \leq  c_{\epsilon}^* + \mathcal{K}_c M \delta.		
	\label{eqn:CostEpsSimiar}
	\end{equation}		
	We next derive an upper bound on the number of trajectory segments, $M$, using the lower bound on the cost of the trajectory connecting two consecutive samples, $\qvect_{m-1}$ and $\qvect_{m}$. To this purpose, let us define a smaller bounding box centered at $\qvect_{m}$ for each sample, which we shall denote as ${\tt BB}(\boldsymbol{\pi}_{m})$, such that the frontier of ${\tt BB}(\boldsymbol{\pi}_{m})$ intersects with $\sigma^\star_\epsilon$ at $\qvect_{m}$ (see Figure~\ref{fig:bbox seqn}) and all its sides are obtained by decreasing those of {\tt BoundingBox}$(\boldsymbol{\pi}_{m})$  of $2\bar \epsilon$ evenly. Then, we can define a minimum cost $l_{min}$ among all the trajectories that reach the frontier of ${\tt BB}(\boldsymbol{\pi}_0)$ from $\boldsymbol{\pi}_0$. Note that $l_{min}>0$ since the length of the shortest side of ${\tt BB}(\boldsymbol{\pi}_0)$ is larger than zero and we assumed non-zero cost for trajectories joining two different states.
Then, the cost $C(e_{\qvect_{m-1} , \qvect_{m}})$ of the trajectory connecting $\qvect_{m-1}$ and $\qvect_{m}$, for $m=1, \dots, M-1$, satisfies the following inequality 	
	\begin{equation}
	l_{min} \leq C(e_{\qvect_{m-1} , \qvect_{m}}).
	\label{eqn: lower bound on l}
	\end{equation}	
Note that we have to treat separately the last trajectory segment connecting $\qvect_{M-1}$ to $\qvect_{M}$ since the existence of an $\bar{\epsilon}$-ball centered at $\qvect_M$ and contained within ${\tt BoundingBox}\left( \boldsymbol{\pi_{M-1}}\right) $ can not be guaranteed as $\qvect_{M}$ is the final state which is fixed. Therefore, \eqref{eqn: lower bound on l} may not hold for the last trajectory segment. We can then rewrite equation \eqref{eqn:cost_segments} as follows 	
	\begin{equation*}
	c_{\epsilon}^* = \sum_{m=1}^{M-1} C(e_{\qvect_{m-1} , \qvect_{m}}) + C(e_{\qvect_{M-1},\qvect_{M}})
	\end{equation*}
	and lower bound the cost of $\sigma_\epsilon^\star$  as
	\begin{equation*}
	c_{\epsilon}^* \geq (M-1)l_{min} + 0.
	\end{equation*}
	Considering that the number of trajectory segments can be upper bounded by
	\begin{equation*}
	M-1 \leq \dfrac{c_{\epsilon}^*}{l_{min}},
	\end{equation*}
	from \eqref{eqn:CostEpsSimiar} it follows that
	\begin{equation*}
	c^{\Delta}_{\epsilon} \leq  \left(1 + \frac{\mathcal{K}_c\delta}{l_{min}}\right)  c_{\epsilon}^* + \mathcal{K}_c \delta.
	\end{equation*}
	Substituting the definition of $\delta$ given in \eqref{eqn:deltaDefinition}, we get
	\begin{equation*}
	c^{\Delta}_{\epsilon} \leq  \left(1 + \frac{\mathcal{K}_c \, \alpha \, \epsilon}{2 \, \mathcal{K}_f \, l_{min}}\right)  c_{\epsilon}^* + \dfrac{\mathcal{K}_c \, \alpha \, \epsilon}{2 \, \mathcal{K}_f},
	\end{equation*}
	which shows that the cost of the $\epsilon$-similar trajectory for a specific gridding is upper bounded by the cost of the $\epsilon$-optimal trajectory and the $\epsilon$-clearance.}	
	\BSnote{
	Now the cost $c^{\star\Delta}$ of the optimal trajectory that can be obtained given a particular discretization satisfies
	\begin{equation}
	c^{\star\Delta}\leq c^{\Delta}_{\epsilon}  
	\leq  \left(1 + \frac{\mathcal{K}_c \, \alpha \, \epsilon}{2\mathcal{K}_f \, l_{min}}\right)  c_{\epsilon}^* + \dfrac{\mathcal{K}_c \, \alpha \, \epsilon}{2\mathcal{K}_f} \label{eqn:upper_bound_res_opt_cost}
	\end{equation}	
	which provides an upper bound on the \emph{resolution optimal $\Delta$-cost}, $c^{\star\Delta}$.}
	
\BSnote{As the resolution of the discretization increases and as the grid converges to the continuous state space, $\epsilon$ can converge to zero.
Then, since $\mathcal{K}_f$ and $\mathcal{K}_c$ are constants and $l_{min}$ is fixed (it depends on the upper bound $\bar\epsilon$ on $\epsilon$), from \eqref{eqn:upper_bound_res_opt_cost} it follows that as $\epsilon$ goes to zero the \emph{resolution optimal $\Delta$-cost} converges to the cost $c^*$ of the optimal trajectory, i.e.,	
	\begin{equation}
	\lim_{\epsilon \rightarrow 0} c^{\star\Delta} = c^*.
	\end{equation}
}

\end{proof}

\BSnote{Theorem~\ref{thm:AsmpOpt} states that when the number of nodes in $\mathcal{G}^\Delta_{free}$ goes to infinity, i.e., when the uniform gridding converges to the continuous state space, the resolution optimal $\Delta$-cost converges to the cost $c^*$ of the optimal trajectory in the continuous state space without gridding. By Theorem~\ref{thm:resOpt}, it then follows that, when the uniform gridding converges to the continuous state space, as the number of iterations goes to infinity the cost of the trajectory returned by \RRTMP\, converges to the cost of the optimal trajectory. Furthermore, Theorem~\ref{thm:AsmpOpt} establishes an upper bound on the \emph{resolution optimal $\Delta$-cost} as a function of the cost of the $\epsilon$-optimal trajectory and the $\epsilon$-clearance such that as the grid resolution increases the upper bound decreases (see equation \eqref{eqn:upper_bound_res_opt_cost}).} However, as shown in  Theorem~\ref{thm:resOpt}, as the number of discrete states increase, \RRTMP\, will take more iterations to return the \emph{resolution optimal $\Delta$-trajectory}. Depending on the problem at hand, one should make the best compromise between the computing time (and size of the database) and the performance in terms of cost.



\begin{figure*}[!t]
	\centering	
	\subfigure[square gridding with coarse resolution]{\label{fig:grid_types-a}\includegraphics[width=.3\linewidth]{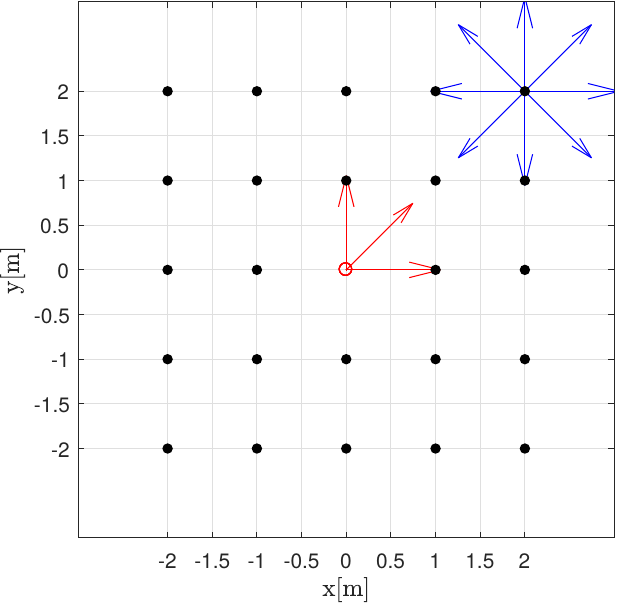}} \hfill
	\subfigure[square gridding with fine resolution]{\label{fig:grid_types-b}\includegraphics[width=.3\linewidth]{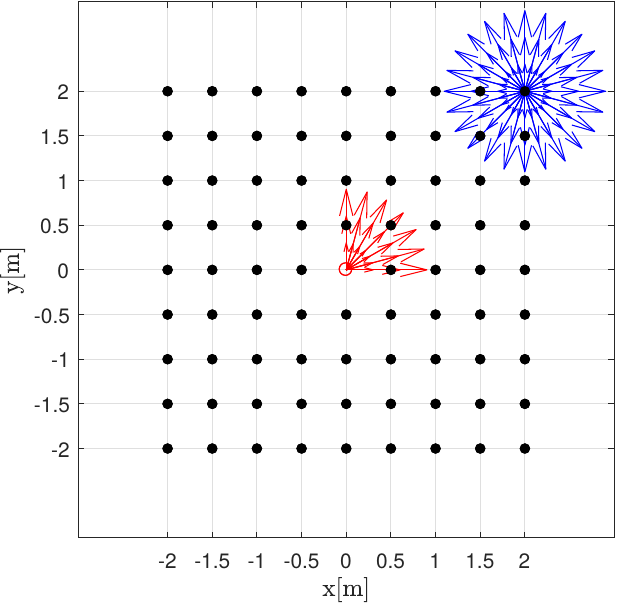}} \hfill
	\subfigure[diamond gridding with fine resolution]{\label{fig:grid_types-c}\includegraphics[width=.3\linewidth]{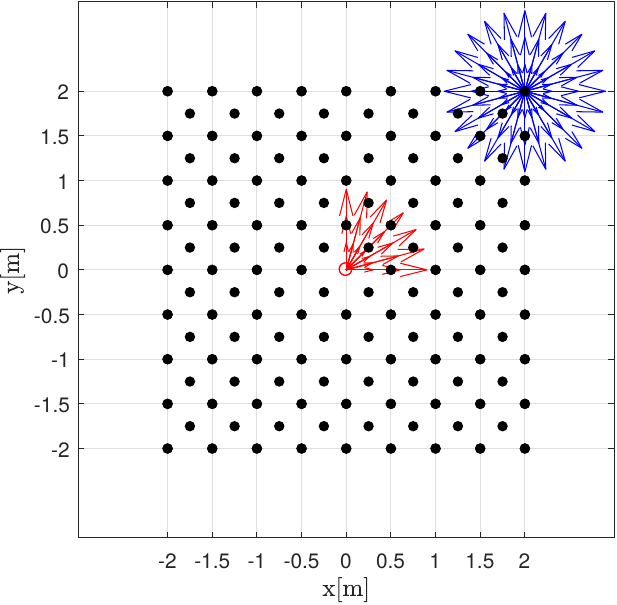}}
	\caption{The three grids used to generate the database. The red circle is the initial position, the blue dots the final ones. Red arrows represent the initial headings and velocities, blue arrows are the final ones. Different arrow sizes correspond to different velocities.\label{fig:grid_types}}
\end{figure*}

\section{NUMERICAL EXAMPLE} \label{sec:results}
In this section a numerical example is presented to show the effectiveness of the proposed algorithm.\\
A 4D state-space $(x,y,\theta,v)$ representing a unicycle like robot moving on a planar surface is considered. The robot is described by the following equations
\begin{equation}
\begin{cases}
\dot{x}(t)=v(t)\cos\theta(t)\\
\dot{y}(t)=v(t)\sin\theta(t)\\
\dot{\theta}(t)=w(t)\\
\dot{v}(t)=a(t)
\end{cases}
\label{eqn:unicycle}
\end{equation}
where $(x,y)$ is the position of the robot and $\theta$ the orientation with respect to a global reference frame, $v$ and $w$ are the linear and angular velocity, respectively. The control input is represented by $\uvect=[w,a]^T$, where $w$ and $a$ are angular velocity and linear acceleration, respectively.

\begin{figure}[hpbt]
	\centering
	\subfigure{\includegraphics[width=.99\linewidth]{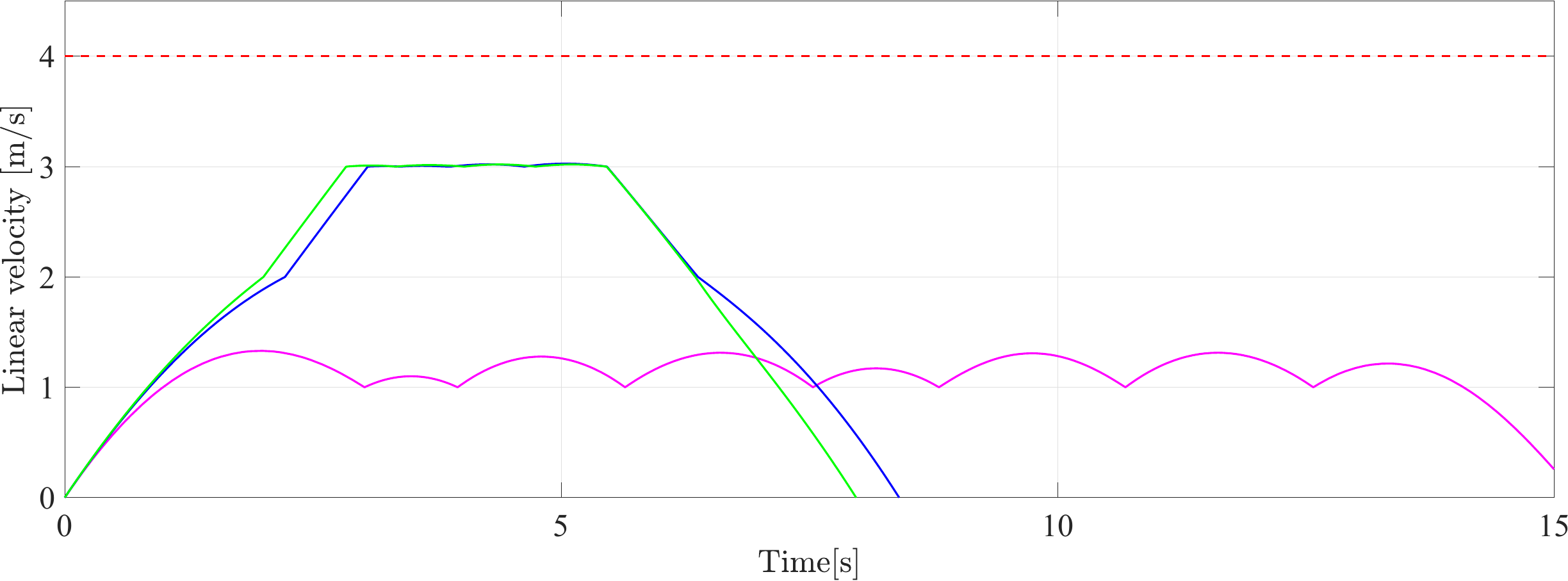}} \\
	\subfigure{\includegraphics[width=.99\linewidth]{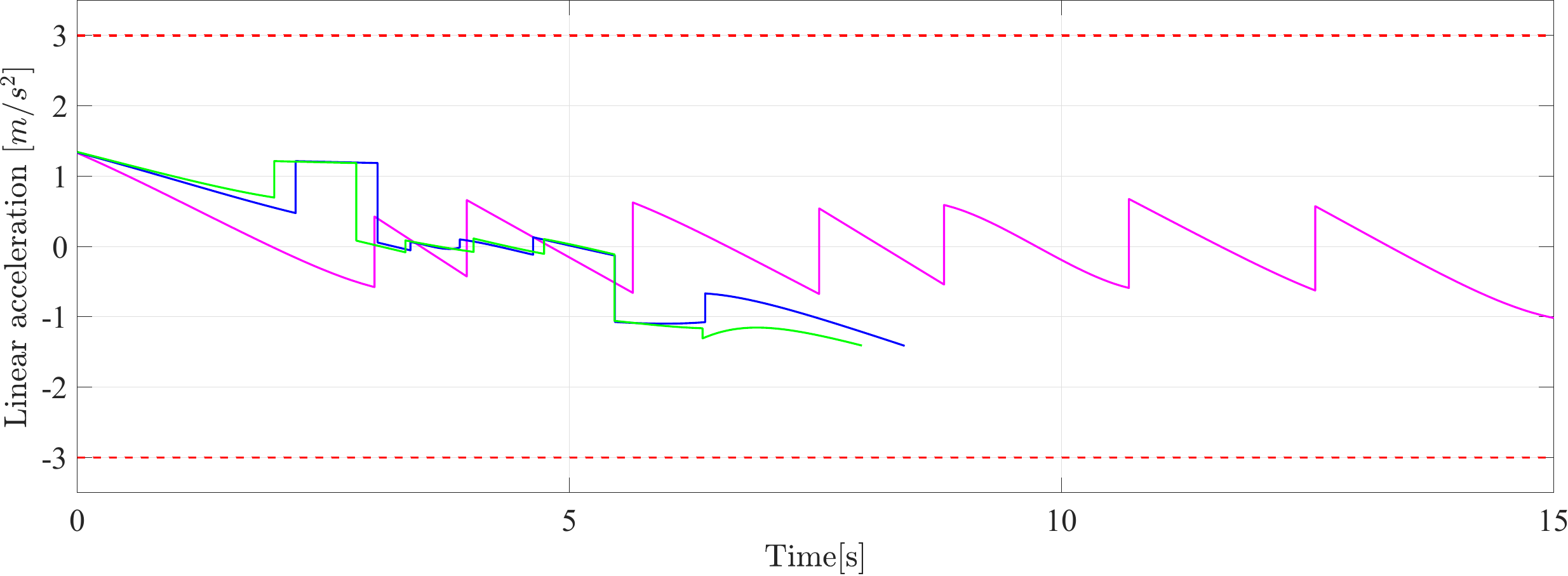}}\\
	\subfigure{\includegraphics[width=.99\linewidth]{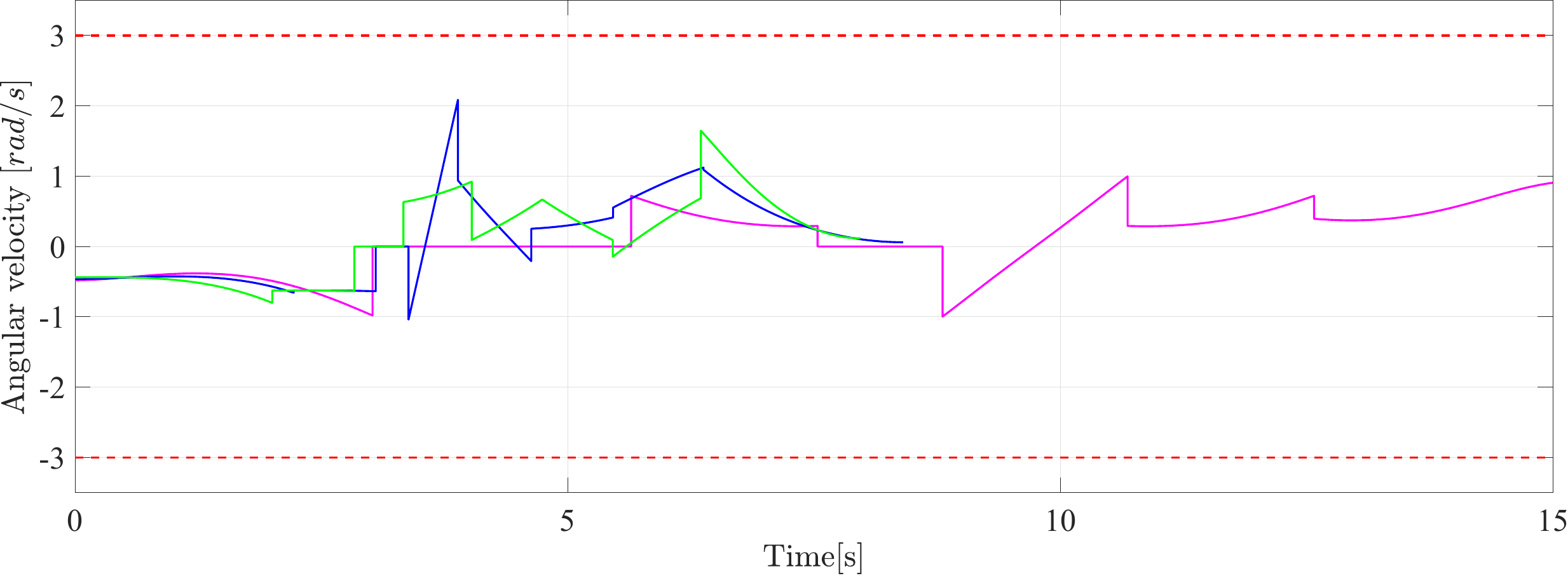}}
	\caption{Linear velocity and actuation profiles for the optimal trajectories in Figures \ref{fig:results_LR_3000} (pink line), \ref{fig:results_FR_50000} (blue line) and \ref{fig:results_DR_50000} (green line). Red lines show the velocity and actuation limits.}
	\label{fig:ResultsActProfiles}
\end{figure}

The motion primitives in the database are computed for each pair of initial and final state, $\qvect_0=[x_0,y_0,\theta_0, v_0]$ and $\qvect_f=[x_f,y_f,\theta_f,v_f]$, solving the TPBVP in \eqref{eqn:cost} for the differential equations given in \eqref{eqn:unicycle} and the cost function
\begin{equation*}
J(\uvect, \tau)=\int\limits_{0}^{\tau}\left[ 1 + \uvect(t)^TR\uvect(t)\right] \mathrm{d}t
\end{equation*}
that minimizes the total time of the trajectory $\tau$, penalizing the total actuation effort with a weight $R=0.5I_2$. The control variables $a$ and $w$ are bounded as $a \in [-3,3]\,\textrm{m/s}^2$, $w \in [-5,5]\,\textrm{rad/s}$.\\
TPBVPs are solved using MATLAB toolbox GPOPS \citep{patterson2014gpops}, a nonlinear optimization tool based on the Gauss pseudo-spectral collocation method.
\begin{figure*}[hpbt]
	\centering
	\subfigure[1000 iterations]{\includegraphics[width=.32\linewidth, clip,trim=11.4cm 0cm 11.8cm 1.25cm]{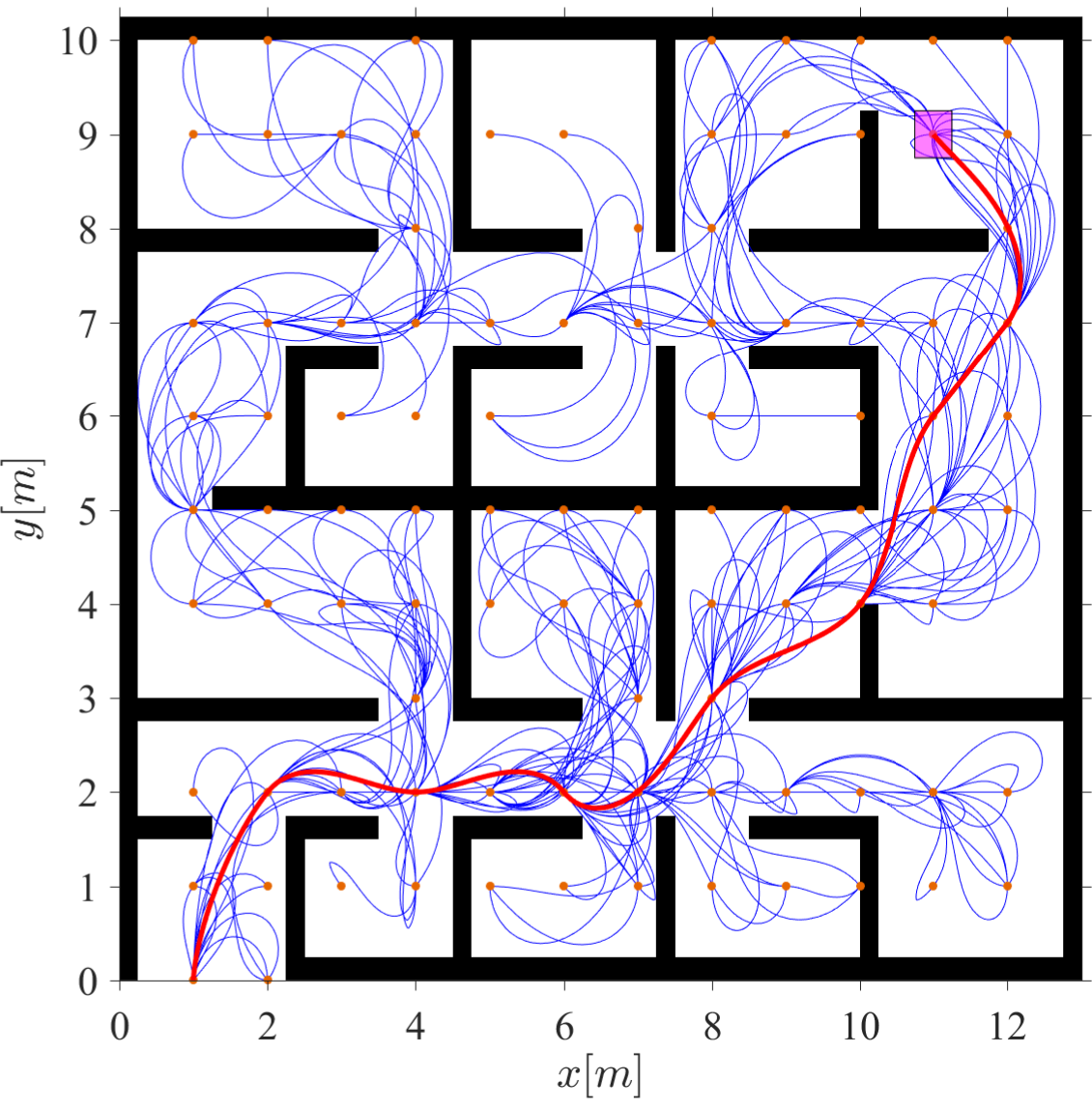}\label{fig:results_LR_1000}} \hfill
	\subfigure[2000 iterations]{\includegraphics[width=.32\linewidth, clip,trim=11.4cm 0cm 11.8cm 1.25cm]{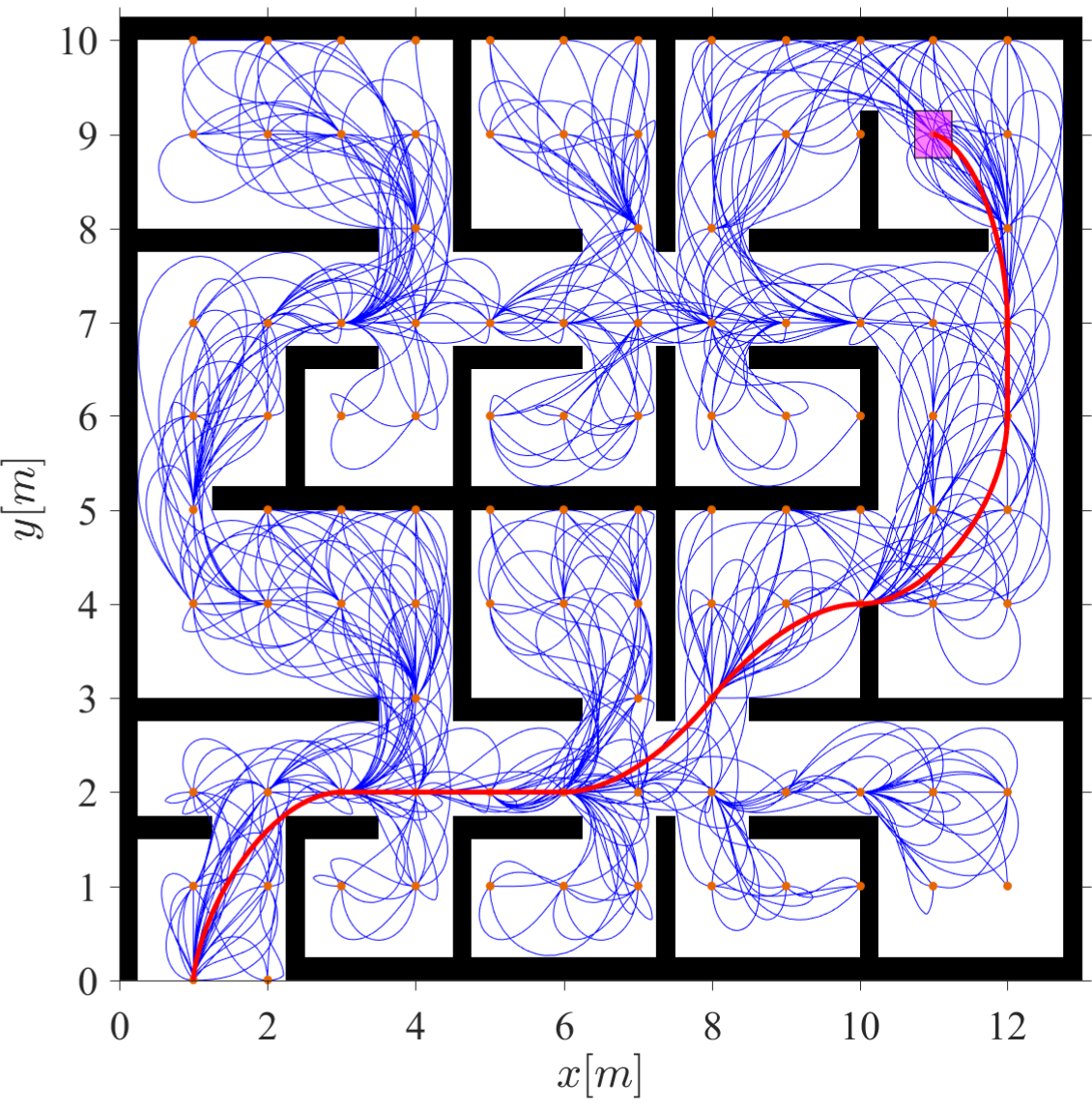}\label{fig:results_LR_2000}} \hfill
	\subfigure[3000 iterations]{\includegraphics[width=.32\linewidth, clip,trim=11.4cm 0cm 11.8cm 1.25cm]{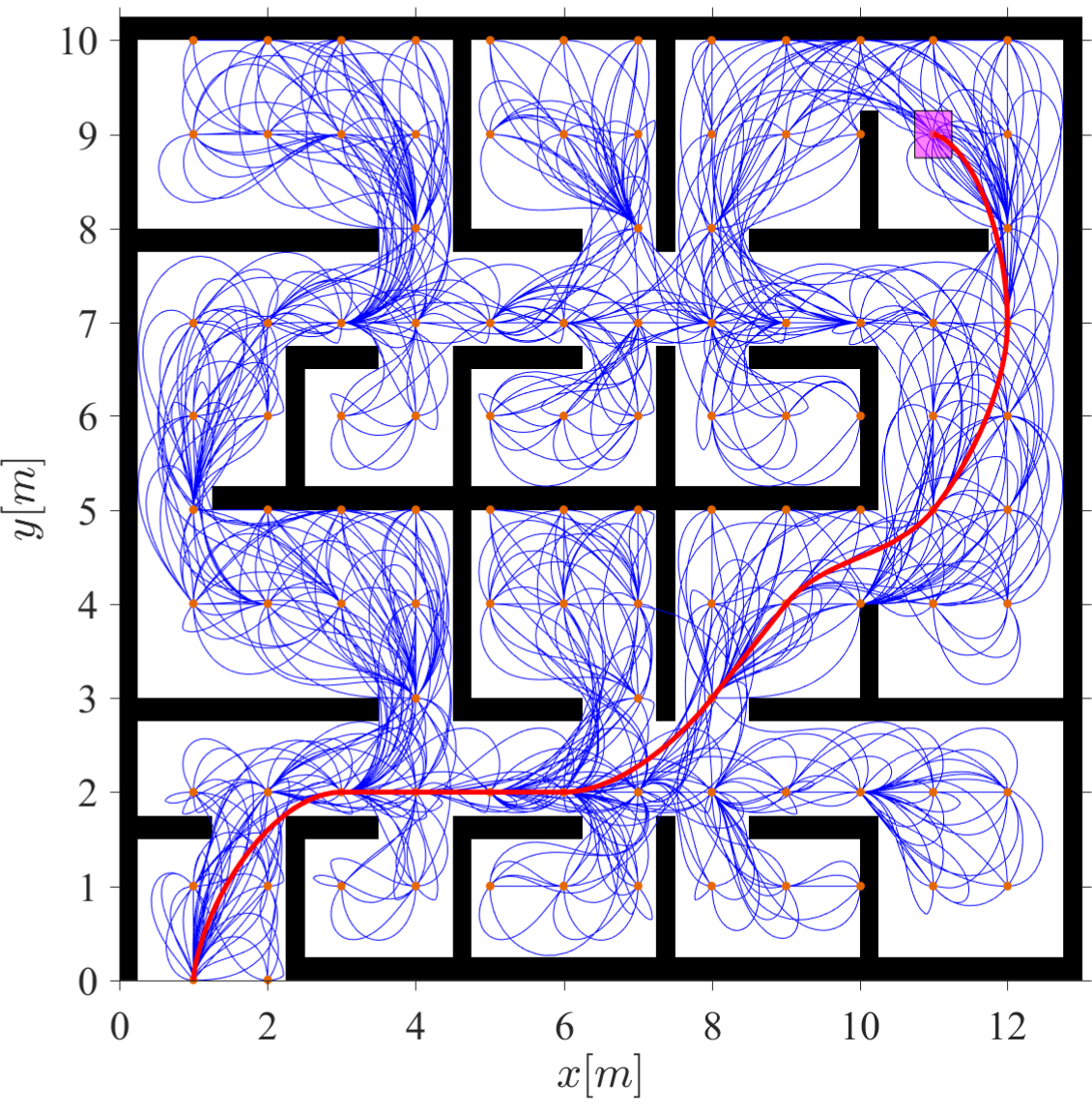}\label{fig:results_LR_3000}}
	\caption{Trajectories generated with an indoor map and a coarse resolution uniform square gridding for various number of iterations. Magenta square is the goal region, and the optimal trajectories are represented in red.}
	\label{fig:results_LR}
\end{figure*}
\begin{figure*}[hpbt]
	\centering
	\subfigure[5000 iterations]{\includegraphics[width=.32\linewidth, clip,trim=11.4cm 0cm 11.8cm 1.25cm]{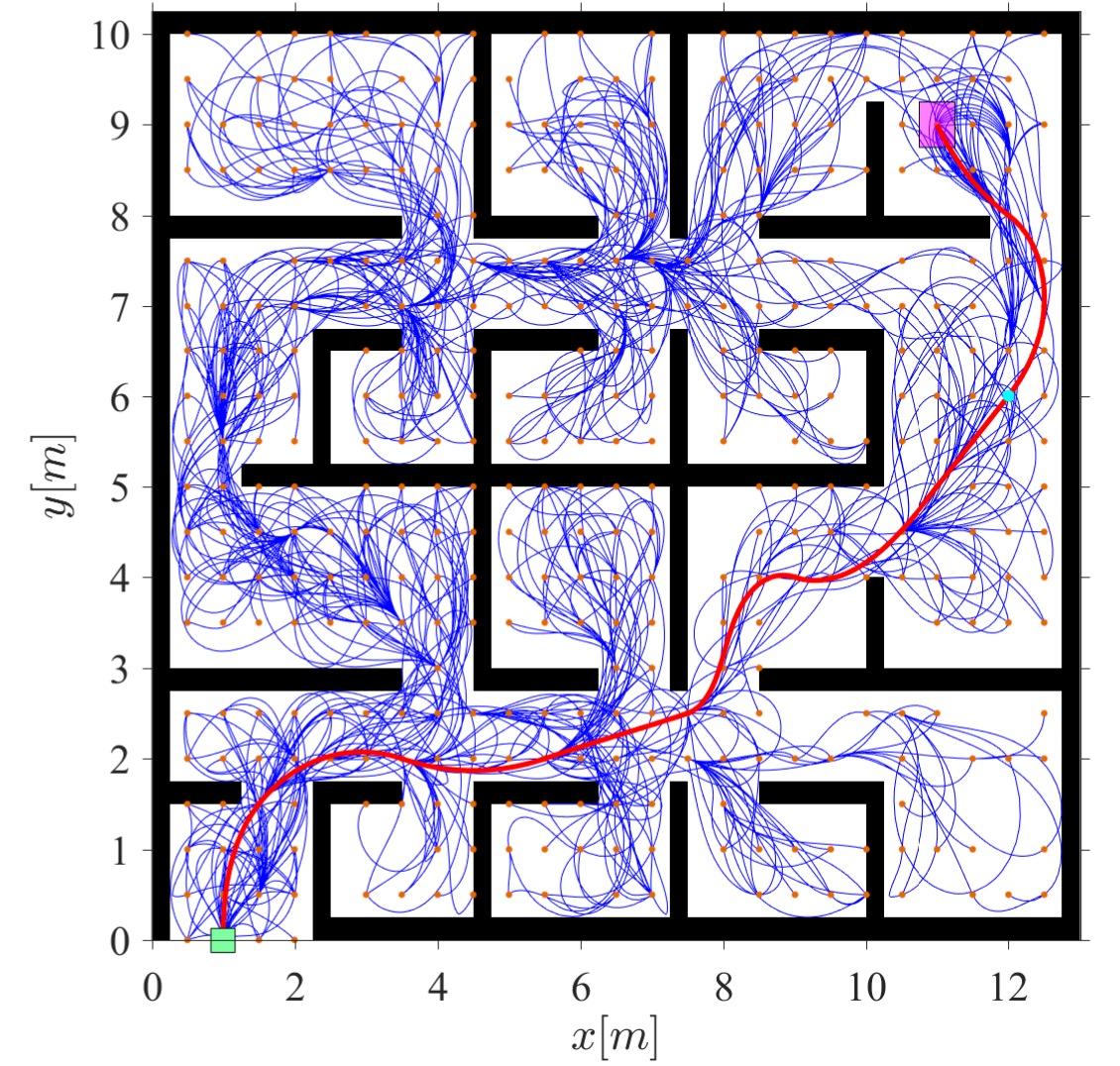}\label{fig:results_FR_5000}} \hfill
	\subfigure[20000 iterations]{\includegraphics[width=.32\linewidth, clip,trim=11.4cm 0cm 11.8cm 1.25cm]{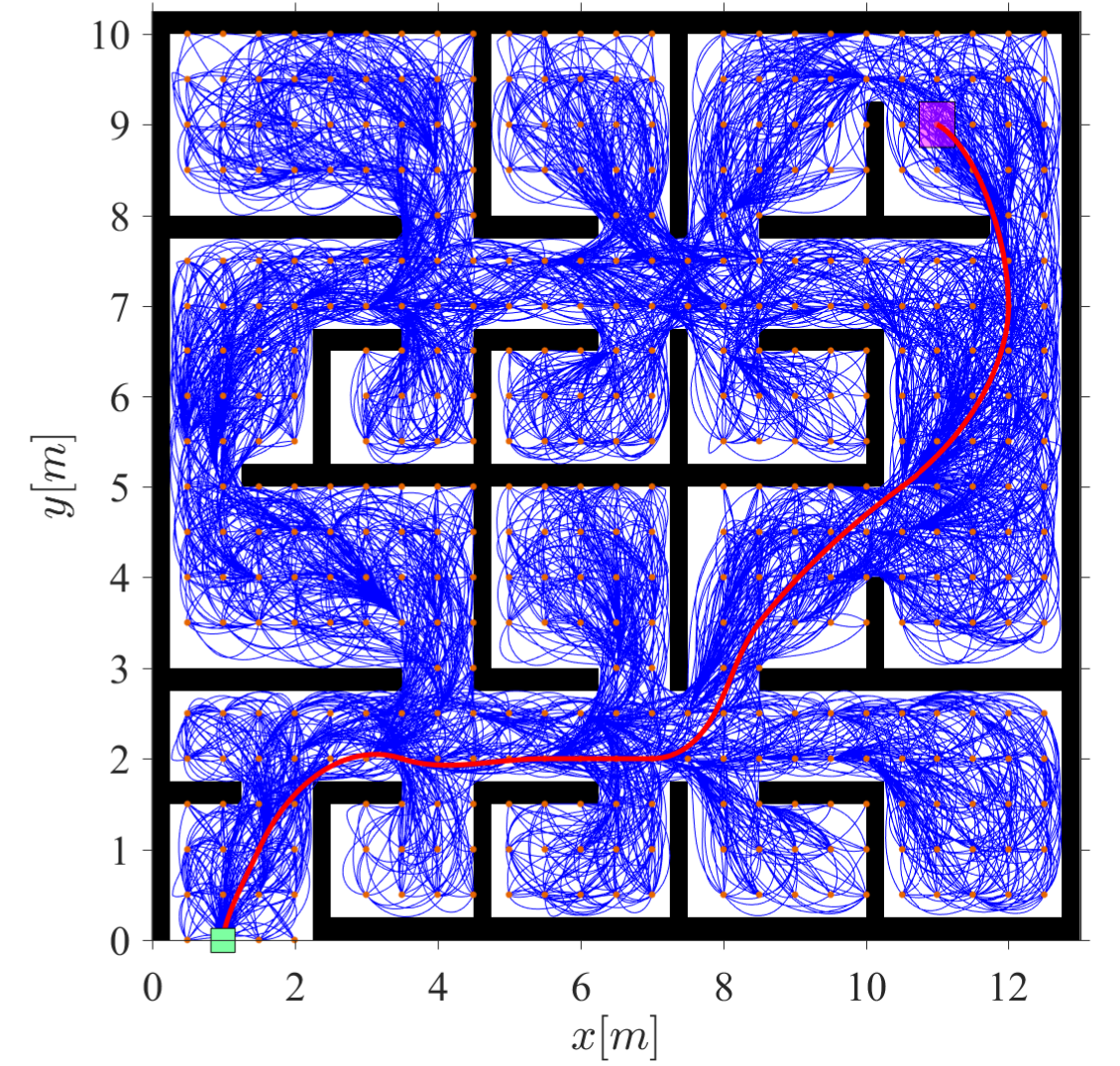}\label{fig:results_FR_20000}} \hfill
	\subfigure[50000 iterations]{\includegraphics[width=.32\linewidth, clip,trim=11.4cm 0cm 11.8cm 1.25cm]{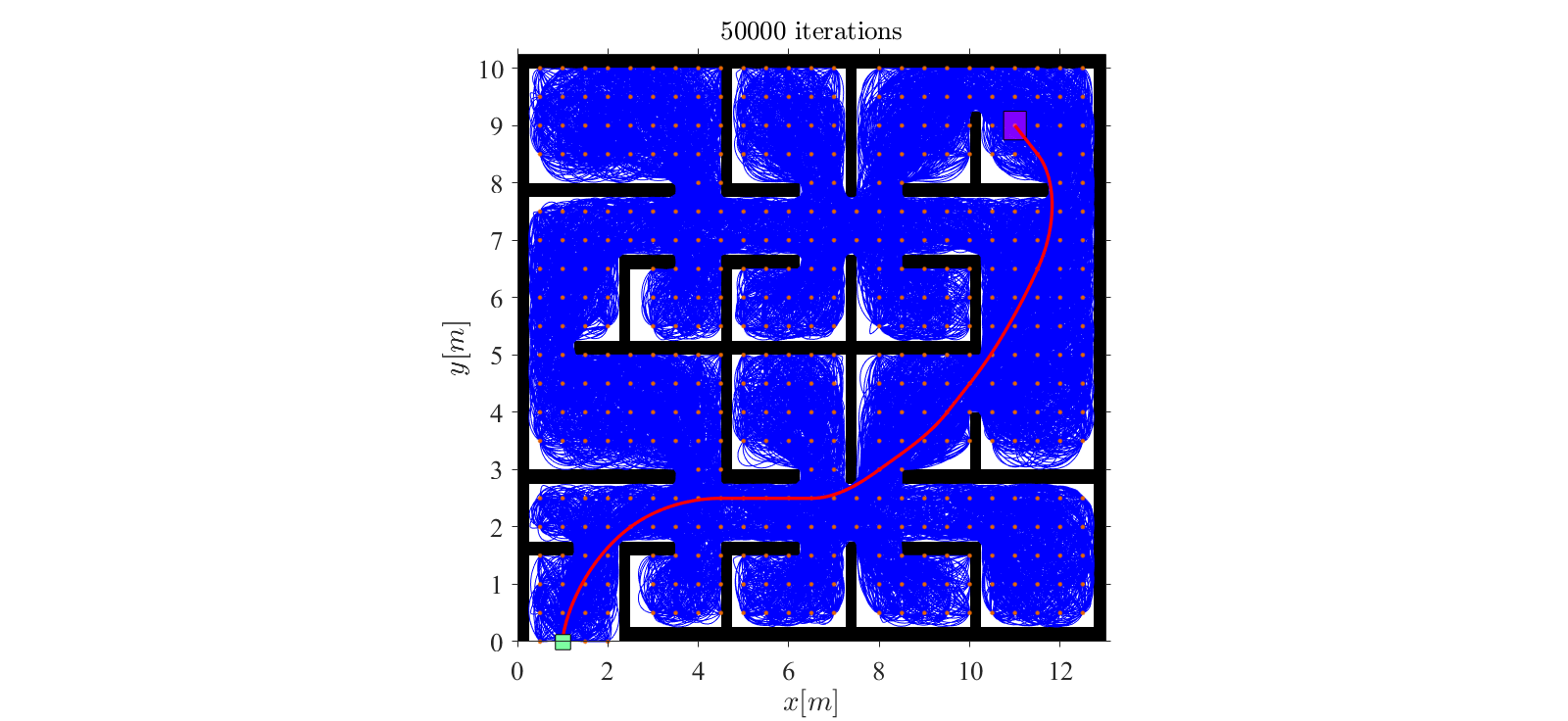}\label{fig:results_FR_50000}}
	\caption{Trajectories generated with an indoor map and a fine resolution uniform square gridding for various number of iterations. Magenta square is the goal region, and the optimal trajectories are represented in red.}
	\label{fig:results_FR}
\end{figure*}
\begin{figure*}[hbpt]
	\centering
	\subfigure[5000 iterations]{\includegraphics[width=.324\linewidth, clip,trim=11.2cm 0cm 11.1cm 1.3cm]{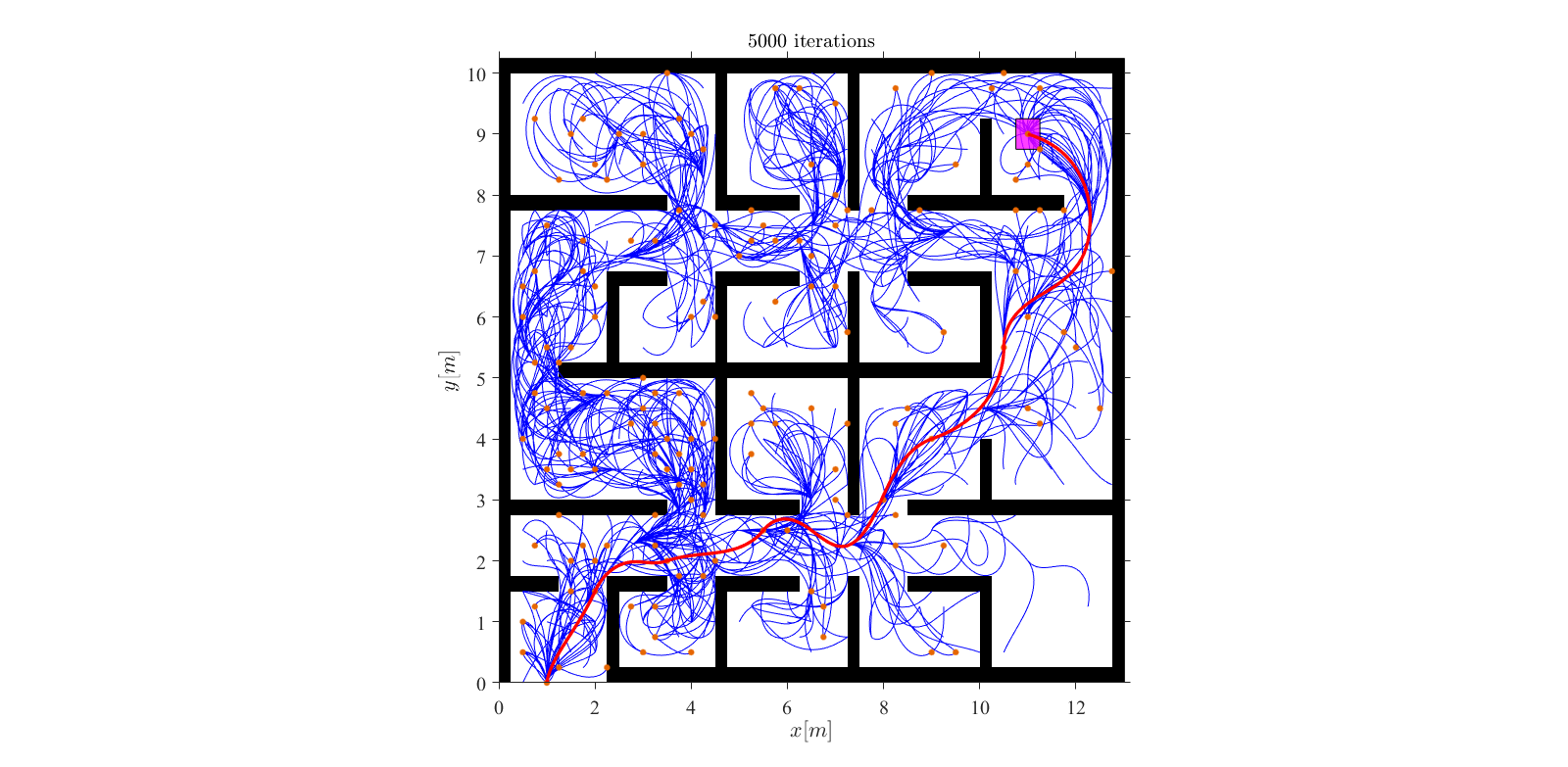}\label{fig:results_DR_5000}} \hfill
	\subfigure[20000 iterations]{\includegraphics[width=.325\linewidth, clip,trim=12.8cm 0cm 12.8cm 1.05cm]{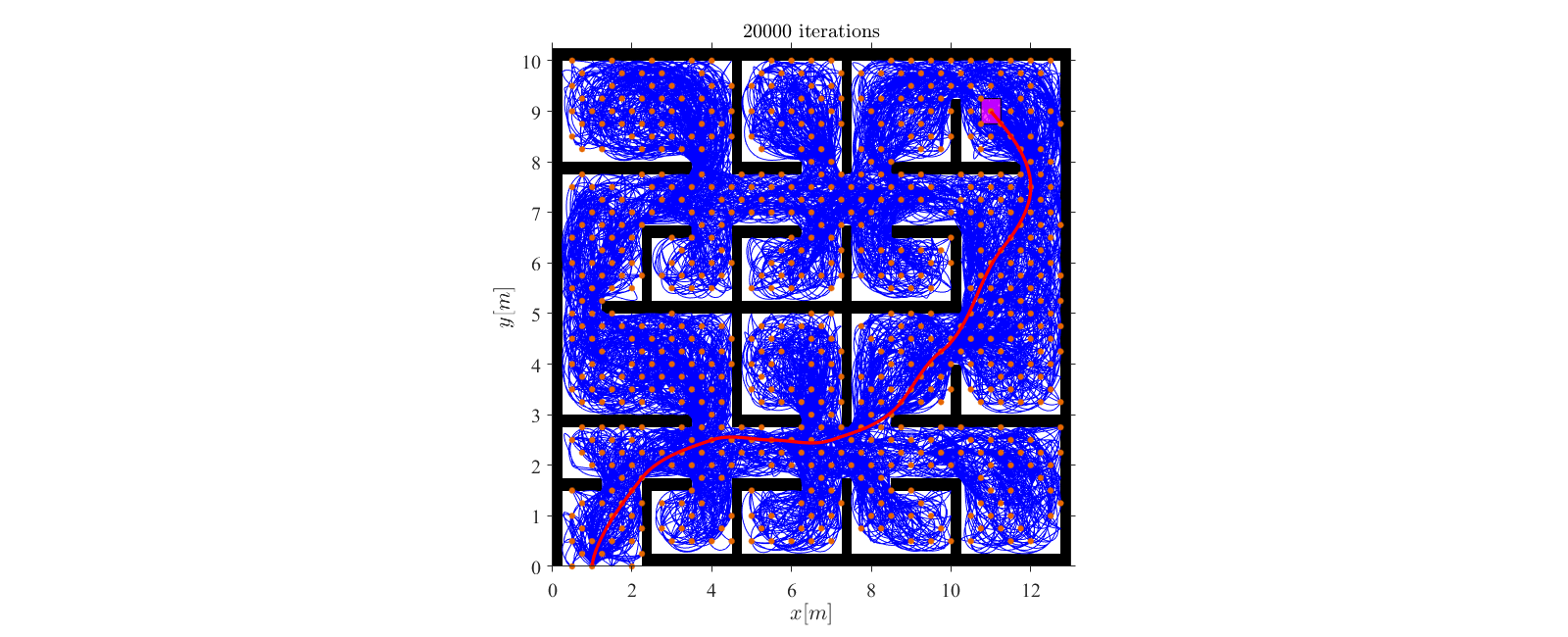}\label{fig:results_DR_20000}} \hfill
	\subfigure[50000 iterations]{\includegraphics[width=.32\linewidth, clip,trim=11.8cm 0cm 11.8cm 1.25cm]{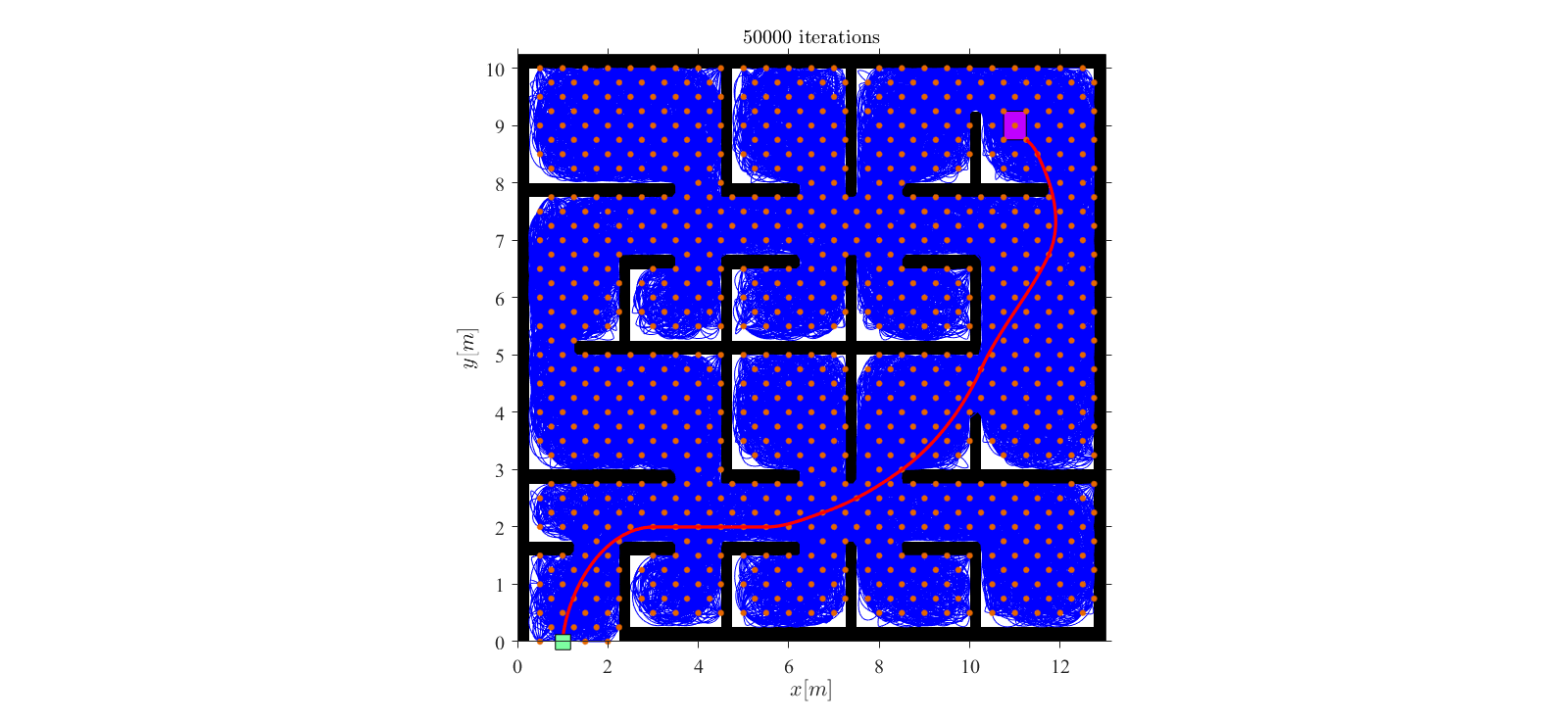}\label{fig:results_DR_50000}}
	\caption{Trajectories generated with an indoor map and a fine resolution uniform diamond gridding for various number of iterations. Magenta square is the goal region, and the optimal trajectories are represented in red.}
	\label{fig:results_DR}
\end{figure*}

Three databases based on different grids have been considered. For all of them the initial state is characterised by the same position $(x_0,y_0) = (0,0)$.\\
The first database is based on a coarse resolution uniform square grid (Figure \ref{fig:grid_types-a}), where \BSnote{$(x_f,y_f) \in [-2,2] \times [-2,2] \setminus \{(0,0)\}$} and each square cell has a size of one meter. The initial orientation $\theta_0$ is selected among three values $\{0, \pi/4, \pi/2 \}\,\textrm{rad}$, the final orientation $\theta_f$ can take 8 equally spaced values in the range $[0,2\pi)\,\textrm{rad}$. For the initial and final velocities, $v_0$ and $v_f$, a minimum and a maximum velocity of $1\,\textrm{m/s}$ and $4\,\textrm{m/s}$ is considered.\\
The second database is based on a fine resolution uniform square grid (Figure \ref{fig:grid_types-b}), where \BSnote{$(x_f,y_f) \in [-2,2] \times [-2,2] \setminus \{(0,0)\}$} and each square cell has a size of half a meter. The initial and final orientations can take 24 equally spaced values in the range $[0,2\pi)\,\textrm{rad}$. The initial and final velocities are selected among 5 equally spaced values in the range $[0,4]\,\textrm{m/s}$.\\
Finally, the third database is based on a uniform diamond grid (Figure \ref{fig:grid_types-c}), characterised by the same initial and final states as the previous one, plus some additional final states at $(x_f,y_f) \in [-1.75,1.75] \times [-1.75, 1.75]$ with a discretization step of half a meter in each direction. These additional states are characterized by the same orientation and velocity of the rest of the database.
\begin{figure}[htbp]
	\centering
	\includegraphics[width=.99\linewidth]{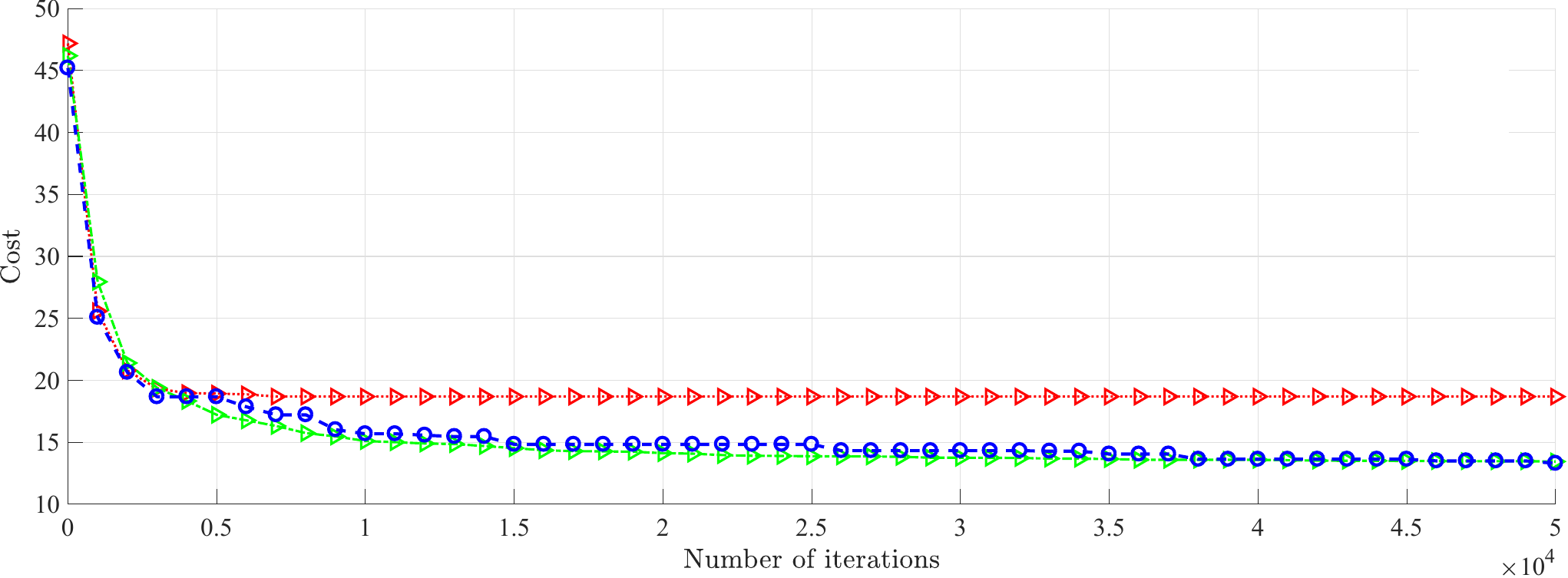}
	\caption{Cost with respect to the number of iterations. Coarse resolution square gridding (red line), fine resolution square gridding (green line) and diamond gridding (blue line) as a function of number of iterations.}
	\label{fig:costs_all}
\end{figure}

Simulations are performed on an IntelCore i7@2.40 GHz personal computer with 8Gb RAM and the algorithm has been implemented in MATLAB.\\
An indoor map is considered (Figures \ref{fig:results_LR}-\ref{fig:results_DR}), setting the robot initial pose at $\left(1,0,\pi/2\right)$ with zero velocity. The goal area is defined as a square of half a meter side and centred at $(9,11)$. The robot should stop at the end of the trajectory.
\begin{figure}[htbp]
	\centering
	\includegraphics[width=.99\linewidth]{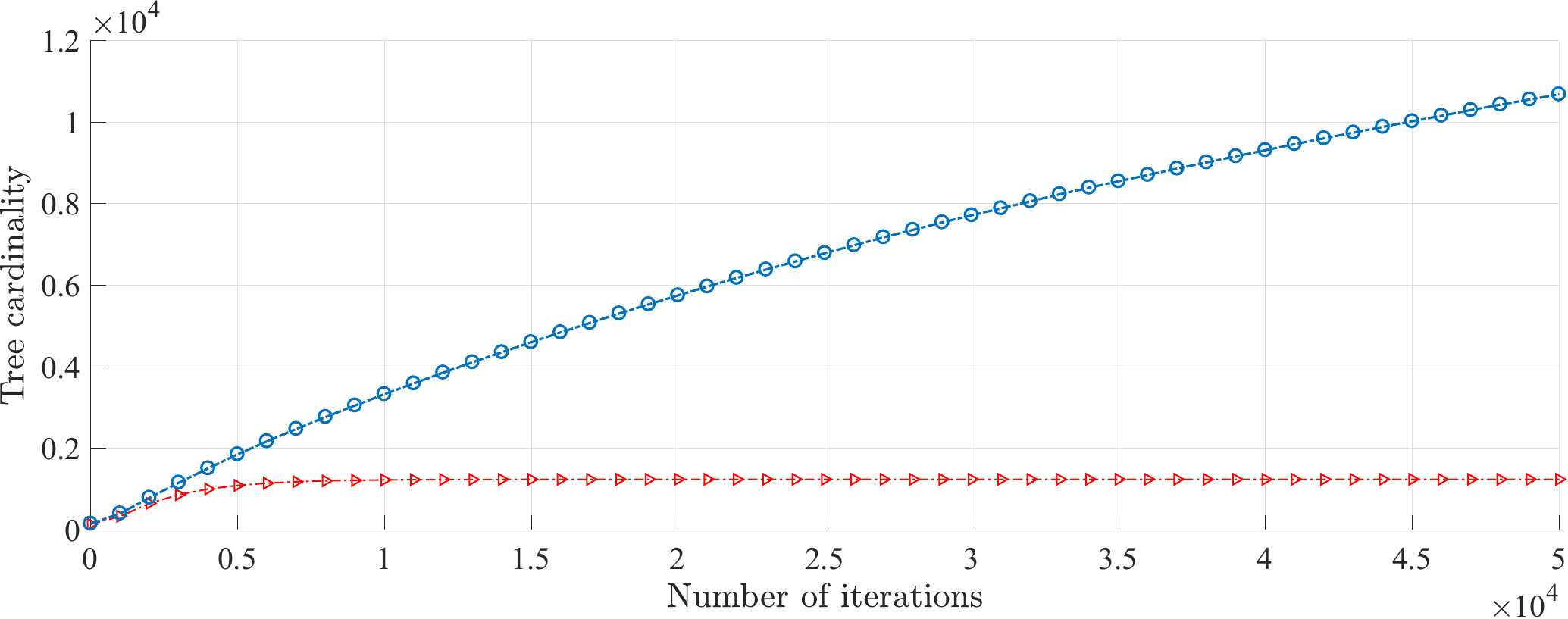}
	\caption{Tree cardinality with respect to the number of iterations, for coarse (red line) and fine (blue line) resolution square gridding.}
	\label{fig:treeCard}
\end{figure}

\begin{figure}[htbp]
	\centering
	\includegraphics[width=.99\linewidth]{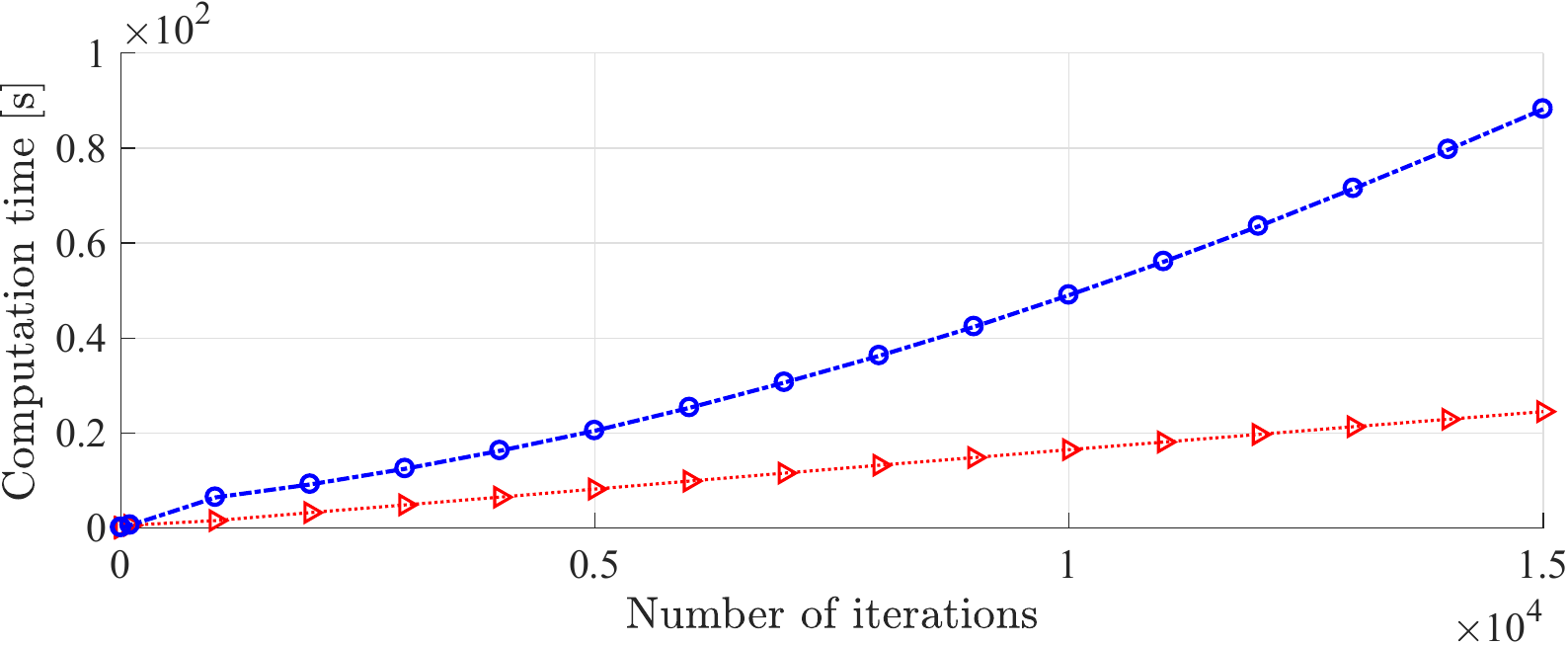}
	\caption{Computation time with respect to the number of iterations, for coarse (red line) and fine (blue line) resolution square gridding.}
	\label{fig:elapsedTime}
\end{figure}

Figures~\ref{fig:results_LR}-\ref{fig:results_DR} show the trees and the optimal trajectories obtained for different number of iterations and for the three different gridding strategies. Correspondingly, Figure \ref{fig:ResultsActProfiles} shows the velocity and the actuation profiles for the optimal trajectories computed with 3000 and 50000 iterations, and reported in Figures~\ref{fig:results_LR_3000}, \ref{fig:results_FR_50000} and \ref{fig:results_DR_50000}, clearly demonstrating that the velocity constraint and the actuation bounds are satisfied. Note that the velocity profile that corresponds to the coarse resolution square gridding exhibits a jerky acceleration behaviour, due to the fact that the velocity at each node is constrained to be exactly one of the values in the database. This demonstrates that the velocity discretization step has to be accurately selected if a smoother velocity profile is required.

The same planning problem has been solved for 10 independent simulation runs. Figure~\ref{fig:costs_all} shows the average cost evolutions related to the coarse resolution, fine resolution and diamond gridding, as the number of iterations increases. As expected, the cost reduces increasing the number of iterations, and converges to the \emph{resolution optimal $\Delta$-cost}: once this minimum is achieved the solution will not further improve.\\
As can be easily seen, motion primitives computed using a denser grid provide lower cost plans. Moreover, the resolution optimal $\Delta$-cost achieved using the fine resolution square and diamond grids are similar, demonstrating that the choice of the discretization step is strictly related to the specific problem. Finally, it is worth mentioning that as the resolution of the grid increases, the cardinality of $Q_{free}^{\Delta}$ increases as well, slowing down the convergence to the resolution optimal $\Delta$-cost.

\BSnote{In order to assess the impact of the grid resolution on the size of the search space, in Table~\ref{tab:gridSizeBranch} we report the number of nodes corresponding to the grids of the three adopted databases (see Figure \ref{fig:grid_types}), together with the corresponding minimum and maximum branching factors, i.e., the number of neighbors that each node is connected to.  The computed number of nodes is an upper bound on the cardinality of $\mathcal{G}^\Delta_{free}$. Yet, from the figures in Table~\ref{tab:gridSizeBranch}, it should be clear the combinatorial nature of the problem, which makes it hard building the whole graph of motion primitives and applying a graph search. As a matter of fact, the most commonly used lattice-based approaches use graph search algorithms that resort to some  heuristic (see for example dynamic A$^\star$ (D$^\star$) and anytime repairing A$^\star$ (ARA$^\star$) by \cite{stentz1994optimal} and \cite{likhachev2008anytime} respectively). 
Note that  \RRTMP \, does not need to adopt any heuristic, but if a smart heuristic were available, it could be integrated within \RRTMP \, to speed up the search.}

\BSnote{\begin{table}[htbp!]
	\vspace{0.1cm}
	\centering \caption{\BSnote{Number of nodes of the state space grids in Figure \ref{fig:grid_types} corresponding to different resolutions together with the minimum and maximum branching factors.} 
	}
	\label{tab:gridSizeBranch}
	\vskip 0.1cm
	\scalebox{0.80}{
		\begin{tabular}{cccc}
			\toprule
			Type of grid & Number of nodes & \multicolumn{2}{c}{Branching factor} \\
			\bottomrule
			&  & min & max \\
			\cline{3-4}
			Figure \ref{fig:grid_types}(a) & 2,460 & 324 & 370 \\
			Figure \ref{fig:grid_types}(b) & 68,000 & 7,300 & 9,100 \\
			Figure \ref{fig:grid_types}(c) & 131,000 & 12,000 & 15,000 \\
			\bottomrule
			& & & 	
		\end{tabular}
	}
\end{table}}

Figure~\ref{fig:treeCard} shows the cardinality of the tree for the coarse and fine resolution square griddings, as the number of iterations increases. Since the nodes that can be added to the tree are limited to the elements of the grid, the tree cardinality converges to the cardinality of $\mathcal{G}^\Delta_{free}$: once this value is achieved no more nodes can be added.

Figure~\ref{fig:elapsedTime} shows the computation time evolution for the coarse and fine resolution griddings, as the number of iterations increases. Though \RRTMP\, does not necessarily add a new node at each iteration, but can only change the existing connections, showing the computation time evolution with respect to the number of iterations allows to easily relate this quantity to the achieved cost, and thus the degree of sub-optimality of the planned trajectory.

It is also worth mentioning that the resulting computation time is promising, even for online replanning in the case of dynamic and partially known environments. Code optimisation and a C/C++ implementation can be considered for a further speed up.

\BSnote{To better emphasize the advantage of using a precomputed database, we report in Figure~\ref{fig:histogram} the histogram of the computation time for solving a single TPBVP of the considered example using the GPOPS commercial numerical solver \citep{patterson2014gpops}. As can be seen from this figure, it typically takes around 400 ms to get a solution for a single TPBVP while a trajectory can be extracted from the database in a time of the order of 0.01 ms (values ranged between 0.008 ms and 0.015 ms over 100 trials). Note that at each iteration of the standard $\tt{RRT^\star}$ algorithm, a set of TPBVPs that corresponds to the set of tentative trajectories connecting $\qvect_{rand}$ to a set of nearby nodes has to be solved. When the TPBVP is not easy to be solved (like in the considered example), the applicability of  $\tt{RRT^\star}$ to dynamic and partially known environments is hampered.
}

\begin{figure}[thbp!]
	\centering
	\includegraphics[trim={1.5cm 0 1.5cm 0},clip,width=\linewidth]{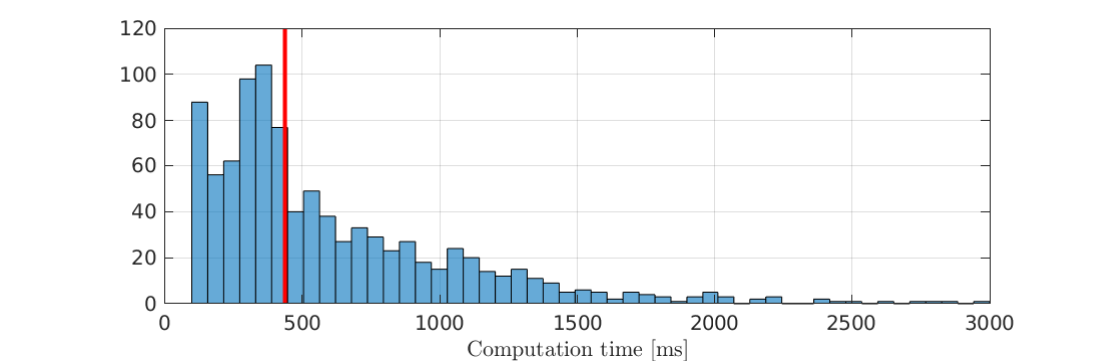}
	\caption{\BSnote{Histogram of the computation time for solving a single TPBVP. The histogram is determined based on 1000 boundary pairs randomly selected from the same subspace used for building the motion primitives.}}
	\label{fig:histogram}
\end{figure}


\section{CONCLUSIONS} \label{sec:conc}
In this paper, a variant of RRT$^\star$, named ${\tt {RRT^\star \text{\underline{\hspace{0.2cm}}} Motion-}}$\
${\tt Primitives}$, that allows to introduce motion primitives in the RRT$^\star$ planning framework is presented. In particular, a set of pre-computed trajectories, named motion primitives, is used to substitute the computationally challenging step of solving for a steering action. Then, in order to ensure that for any queried steering action a pre-computed trajectory exists, a grid representation of the state space has been introduced.\\
This newly conceived algorithm is supported by an accurate theoretical analysis, demonstrating the optimality and probabilistic completeness.\\
The performance of \RRTMP\, has been verified in simulation, showing promising results in terms of quality of the planned trajectory and computation time, that is particularly important for an online usage in the case of dynamic environments that require repeated replanning. The results show also that as the grid size gets smaller, asymptotic optimality is achieved. Having a fine resolution, however, increases the size of the database and the number of iterations required to converge to the resolution optimal trajectory. Nevertheless, one advantage of adopting a sampling based approach is the possibility of computing a feasible though sub-optimal solution first, and then, in case more time is available, improve it. One should indeed choose the best compromise between computing time and performance, according to the application at hand.

\begin{acknowledgements}
This work is supported by the European Commission under the project UnCoVerCPS with grant number 643921.
\end{acknowledgements}

\bibliographystyle{spbasic}      
\bibliography{planning}{}   

\end{document}